\documentclass[twoside,11pt]{article}

\usepackage{jmlr2e}

\usepackage[utf8]{inputenc} 
\usepackage[T1]{fontenc}    
\usepackage{hyperref}       
\usepackage{url}            
\usepackage{booktabs}       
\usepackage{amsfonts}       
\usepackage{nicefrac}       
\usepackage{microtype}      
\usepackage{times}
\usepackage{graphicx} 
\usepackage{subcaption}
\usepackage{tabularx}
\usepackage{bbm}
\usepackage{url}
\usepackage{amsmath}
\usepackage{bbm}
\usepackage{caption}
\usepackage{float}
\usepackage{amssymb}
\usepackage{grffile}
\usepackage{epstopdf}
\usepackage[colorinlistoftodos]{todonotes}
\usepackage{multirow}

\newcommand{\argmin}[1]{\underset{#1}{\operatorname{argmin}}}
\newcommand{\argmax}[1]{\underset{#1}{\operatorname{argmax}}}

\newcommand{\R}{{\mathbb{R}}}
\newcommand{\E}{{\mathbb{E}}}

\newcommand{\cc}{\lambda}
\newcommand{\F}{\cal{F}}

\newcommand{\C}{\mathcal{C}}
\newcommand{\Ft}{\F_\lambda}
\newcommand{\M}{\mathcal{M}}
\renewcommand{\S}{\mathcal{S}}
\newcommand{\MRE}{MRE }
\newcommand{\MREC}{MRE-C }

\renewcommand{\S}{\mathcal{S}}
\newcommand{\Exp}{\mathbb{E}}
\newcommand{\one}{\mathbf{1}}
\newcommand{\var}{\operatorname{var}}

\usepackage{color}

\usepackage[normalem]{ulem}

\newcommand{\hide}[1]{}

\newtheorem{assumption}{Assumption}
\newtheorem{claim}{Claim}
\newtheorem{insight}{Insight}

\begin{document}

\title{One-Shot Federated Learning: \\ Theoretical Limits   and Algorithms to Achieve Them 
\thanks{Parts of this work (including weaker versions of Theorems~\ref{th:main upper c} and~\ref{th:H2constb}) are presented in \cite{AlSG19nips} at Neurips 2019.}}

\author{\name Saber Salehkaleybar \email saleh@sharif.edu \\
	\AND
	\name Arsalan Sharifnassab \email a.sharifnassab@gmail.com \\
	\AND
	\name S.~Jamaloddin Golestani \email golestani@sharif.edu \\
	\AND
	\addr 	Department of Electrical Engineering\\
	Sharif University of Technology \\
	Tehran, Iran}
	
\editor{ }

\maketitle

\begin{abstract}
We consider distributed statistical optimization in one-shot setting, where there are $m$ machines each observing  $n$ i.i.d. samples. 
Based on its observed samples, each machine sends a $B$-bit-long message to a server. The server 
then collects messages from all machines, and estimates a parameter that minimizes an expected convex
loss function. 
We investigate the impact of communication constraint, $B$, on the expected error and derive a tight lower bound on the error achievable by any algorithm. 
We then propose an estimator, which we call \emph{Multi-Resolution Estimator} (\MRE\!), whose expected error (when $B\ge\log mn$) meet the aforementioned lower bound up to poly-logarithmic factors, and is thereby order optimal. 
We also address the problem of learning under tiny communication budget, and present lower and upper error bounds when $B$ is a constant.
The expected error of \MRE, unlike existing algorithms, tends to zero as the number of machines ($m$) goes to infinity, even when the number of samples per machine ($n$) remains upper bounded by a constant. 
This property of the \MRE algorithm makes it applicable in new machine learning paradigms where $m$ is much larger than $n$. 
\end{abstract}
\begin{keywords}
	Federated learning, Distributed learning, Few shot learning, Communication efficiency, Statistical Optimization. 
\end{keywords}


\medskip
\section{Introduction}
 In recent years, there has been a growing interest in various learning tasks over large scale data generated and collected via smart phones and mobile applications. In order to carry out a learning task over this data, a naive approach is to collect the data in a centralized server which might be infeasible or undesirable due to communication constraints or privacy reasons. For learning statistical models in a distributed fashion, several works have focused on designing communication-efficient algorithms for various machine learning applications \citep{duchi2012dual,braverman2016communication,chang2017distributed,diakonikolas2017communication,lee2017communication}.


In this paper, we consider the problem of statistical optimization in a distributed setting as follows. 
Consider an unknown distribution $P$ over a collection, $\F$, of differentiable convex functions with Lipschitz first order derivatives, defined over a convex region in $\R^d$. 
There are $m$ machines, each observing $n$ i.i.d sample functions from $P$. Each machine processes its observed data, and transmits a signal of certain length to a server. 
The server then collects all the signals and outputs an estimate of the parameter $\theta^*$ that minimizes the expected loss, i.e., $\min_{\theta}\mathbb{E}_{f\sim P}\big[f(\theta)\big]$.
See Fig.~\ref{Fig:system} for an illustration of the system model.

We focus on the  distributed aspect of the problem considering arbitrarily large number of machines ($m$)  and present tight lower bounds and matching upper bounds on the estimation error. 
In particular,
\begin{itemize}
\item 
Under general communication budget with $B\ge d\log mn$ bits per transmission, we present a tight lower bound and an order-optimal estimator that achieves this bounds up to poly-logarithmic factors. More specifically, we show that $\|\hat\theta-\theta^*\|=\tilde{\Theta}\left(\max\left(n^{-1/2} (mB)^{-1/d},\,(mn)^{-1/2} \right)\right)$.
\item 
For the regime that the communication budget is very small with constant number of bits per transmission, we present upper and lower bounds on the estimation error and show that  the  error can be made arbitrarily small if $m$ and $n$ tend to infinity simultaneously.
\item 
Compared to the previous works that consider function classes with
Lipschitz continuous second or third order derivatives, our algorithms and bounds are designed and derived for a broader class of functions with Lipschitz
continuous first order derivatives.
This brings our model closer to real-world learning applications where the  loss landscapes involved are  highly non-smooth.
\end{itemize}

\begin{figure}[t]
	\centering
	\includegraphics[width=6.5cm]{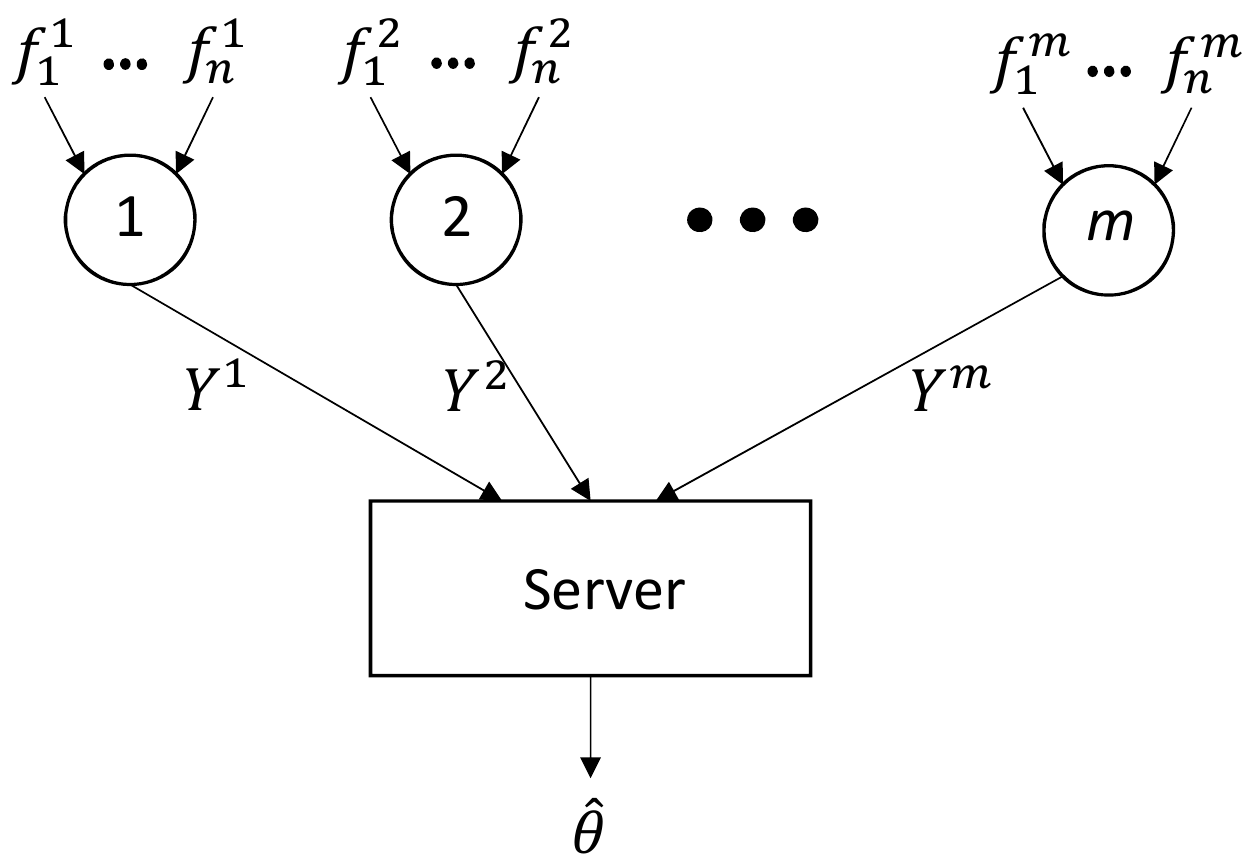}
	\caption{
		A distributed system of $m$ machines, each having access to $n$ independent sample functions from an unknown distribution $P$. 
		Each machine sends a signal to a server based on its observations. The server receives all signals  and output an estimate $\hat{\theta}$ for the optimization problem in \eqref{eq:opt problem}. }
	\label{Fig:system}
\end{figure}

\subsection{Background}
The distributed setting considered here has recently employed in   a new machine learning paradigm called \emph{Federated Learning} \citep{konevcny2015federated}. 
In this framework, training data is kept in users' computing devices due to privacy concerns, and the users participate in the training  process without revealing their data. 
As an example, Google has been working on this paradigm in their recent project, \emph{Gboard} \citep{mcmahan2017federated}, the Google keyboard. 
Besides communication constraints, one of the main challenges in this paradigm is that each machine has a small amount of data. In other words, the system operates in a regime that $m$ is much larger than $n$ \citep{chen2017distributed}.

A large body of distributed statistical optimization/estimation literature considers ``one-shot" setting, in which each machine communicates with the server merely once 
\citep{zhang2013information}. 
In these works, the main objective is to  minimize the number of transmitted bits, while keeping the estimation error as low as the error of a centralized estimator, in which the entire data is co-located in the server.  

If we impose no limit on the communication budget, then each machine can encode its entire data into a single message and sent it to the server. 
In this case, the sever acquires the entire data from all machines, and the distributed problem reduces to a centralized problem.
We call the sum of observed functions at all machines as the centralized empirical loss, and refer to its minimizer as the centralized solution.
It is part of the folklore that the centralized solution is order optimal and its expected error is $\Theta\big(1/\sqrt{mn}\big)$ \citep{lehmann2006theory,zhang2013information}. 
Clearly, no algorithm can beat the performance of the best centralized estimator. 

\subsubsection{Upper bounds}
\citet{zhang2012communication} studied a simple averaging method where each machine obtains the empirical minimizer of its observed functions and sends this minimizer to the server through an $O(\log mn)$ bit message. Output of the server is then the average of all received  empirical minimizers.
\citet{zhang2012communication} showed that the expected error of this algorithm is no larger than $O\big(1/\sqrt{mn}\,+\, 1/n\big)$, provided that:  1- all functions are convex and twice differentiable with Lipschitz continuous second derivatives, and 2- the objective function $\mathbb{E}_{f\sim P}\big[f(\theta)\big]$ is strongly convex at $\theta^*$. 
Under the extra assumption that the functions are three times differentiable with Lipschitz continuous third derivatives, \citet{zhang2012communication} also present a bootstrap method  
whose expected error is $O\big(1/\sqrt{mn}\,+\, 1/n^{1.5}\big)$. 
It is easy to see that, under the above assumptions, the averaging method and the bootstrap method achieve the performance of the centralized solution if $m\leq n$ and $m\leq n^2$, respectively. 
Recently, \citet{jordan2018communication} proposed to optimize a surrogate loss function using Taylor series expansion. This expansion can be constructed at the server by communicating $O(m)$ number of $d$-dimensional vectors. 
Under similar assumption on the loss function as in \citep{zhang2012communication},  they showed that the expected error of their method is no larger than $O\big(1/\sqrt{mn}+1/n^{{9}/{4}}\big)$. 
It, therefore, achieves the performance of the centralized solution for $m \le n^{3.5}$. 
However, note that when $n$ is fixed, 	 all aforementioned bounds remain lower bounded by a positive constant, even when $m$ goes to infinity. 

In \citep{AlSG19nips}, we relaxed the second order differentiability assumption, and considered a model that allows for convex loss functions that have Lipschitz continuous first order derivatives. There we presented an algorithm (called MRE-C-$\log$) with the communication budget of  $\log mn$ bits per transmission, and proved the upper bound $\tilde{O}\big( m^{-{1}/{\max(d,2)}} n^{-1/2}\big)$  on its estimation error. 
In this work we extend this algorithm to general communication budget of $B$ bits per transmission, for arbitrary values of $B\ge \log mn$. 
We also derive a lower bound on the estimation error of any algorithm. This lower bound meets the error-upper-bound of the \MREC algorithm, showing that the \MREC estimator has order optimal accuracy up to a poly-logarithmic factor.

\subsubsection{Lower bounds}
\cite{shamir2014fundamental} considered various communication constraints and showed that no distributed algorithm can achieve performance of the centralized solution with budget less than $\Omega(d^2)$ bits per machine. 
For the problem of sparse linear regression, \cite{braverman2016communication} proved that any algorithm that achieves optimal minimax squared error, requires to communicate  $\Omega(m\times \min(n,d))$ bits  in total from machines 
to the server. 
Later, \cite{lee2017communication} proposed an algorithm that achieves optimal mean squared error  for the problem of sparse linear regression when $d<n$.

\cite{zhang2013information} derived an information theoretic lower bound on the minimax error of parameter estimation, in presence of communication constraints. 
They showed that, in order to acquire the same precision as the centralized solution for estimating the mean of a $d$-dimensional  Gaussian distribution, the machines require to transmit a least total number of $\Omega\big(md/\log(m)\big)$ bits.
\cite{garg2014communication} improved this bound to $\Omega(dm)$ bits using direct-sum theorems \citep{chakrabarti2001informational}.


\subsubsection{One-shot vs. several-shot models}
Besides the one-shot model, there is another  communication model that allows for several transmissions back and forth between the machines and the server.  Most existing works of this type \citep{bottou2010large,lian2015asynchronous,zhang2015deep,mcmahan2017communication} 
involve variants of stochastic gradient descent, in which the server queries at each iteration the gradient of empirical loss at certain points from the machines.  
The gradient vectors are then aggregated in the server to update the model's parameters. The expected error of such algorithms typically scales as  $\tilde{O}\big({1}/{\sqrt{k}}\big)$, where $k$ is the number of iterations.

The bidirectional communication in the several-shot model  makes it convenient for the server to guide the search by sending queries to the machines (e.g., asking for gradients at specific points of interest). This powerful ability of the model typically leads to more efficient communication for the case of convex loss landscapes. 
However, the two-way communication require the users (or machines) be available during the time of training, so that they can respond to the  server queries in real time. 
Moreover, in such iterative algorithms, the users should be willing to reveal parts of their information asked by the servers. 
In contrast to the several-shot model, in the one-shot setting, because of one-way communication, SGD-like iterative algorithms are not applicable. 
The one-shot setting calls for a totally different type of algorithms and lower bounds.


\subsection{Our contributions}
We study the problem of one-shot distributed learning under milder assumptions than previously available in the literature.
We assume that loss functions, $f\in\F$, are differentiable with Lipschitz continuous first order derivatives.
This is in contrast to the works of \citep{zhang2012communication} and \citep{jordan2018communication} that assume Lipschitz continuity of second or third derivatives. 
  The assumption is indeed  practically
important since the loss landscapes involved
in several learning applications are highly non-smooth.
The reader should have in mind this model differences, when comparing our bounds with the existing results.
See Table~\ref{tab:1} for a summary of our results. 

\begin{center}
\begin{table} 
	\centering
	\begin{tabular}{|p{2.3cm}|c|l|c|}
		\hline
		 Communication \phantom{aa} Budget ($B$) & Assumptions & \hspace{3cm} Result & Ref.\\
		\hline
		\multirow{2}{*}{$B\ge d\log(mn)$} & \multirow{2}{*}{-} &  $ \|\hat\theta-\theta^*\|=\tilde{\Omega}\left(\max\left(\frac{1}{\sqrt{n}\, (mB)^{1/d}},\, \frac{1}{\sqrt{mn}} \right)\right)$& Th.~\ref{th:lower c} \\ \cline{3-4}
		& &  $ \|\hat\theta-\theta^*\|=\tilde{O}\left(\max \left(\frac1{\sqrt{n}\,(mB)^{1/{d}}}, \, \frac{1}{\sqrt{mn}} \right) \right)$& Th.~\ref{th:main upper c} \\  \cline{1-4}
		\multirow{2}{*}{Constant $B$} &  $n=1$ &  $ \|\hat\theta-\theta^*\|={\Omega}\left( 1 \right)$ & Th.~\ref{th:Hinf} \\ \cline{2-4}
		&  {$B=d$} &  $ \|\hat\theta-\theta^*\|=O\left(\frac{1}{\sqrt{n}}+\frac{1}{\sqrt{m}}\right)$ & Th.~\ref{th:H2constb} \\
		\hline
	\end{tabular} 
	\caption{Summary of our results. 
	} \label{tab:1}
\end{table}
\end{center}

We consider a sitting where the loss landscape is convex, and derive a lower bound on the estimation error, under communication budget of $B$ bits per transmission for all $B\ge d\log mn$.
We also propose an algorithm (which we call  Multi-Resolution Estimator for Convex setting (\MREC\!)), and show that its estimation error meets the lower bound up to a poly-logarithmic factor. 
Therefore, \MREC algorithm has order optimal accuracy. 
Combining these lower and upper bounds, we show that for any communication budget $B$ no smaller than $d \log mn$, we have $\|\hat\theta-\theta^*\|=\tilde{\Theta}\left(\max\left(n^{-1/2} (mB)^{-1/d},\,(mn)^{-1/2} \right)\right)$.
Moreover, computational complexity of the \MREC algorithm is polynomial in $m$, $n$, and $d$.
Our results also provide the minimum communication budget required for any estimator to achieve the performance of the centralized algorithm.

We also study a regime with tiny communication budget, where $B$ is bounded by a constant. We show that when $B$ is a constant and $n=1$, the  error of any estimator is lower bounded by a constant, even when $m$ tends to infinity. 
On the other hand, we propose an algorithm with the budget of $B=d$ bits per transmission and show that its estimation error is no larger than $O\left(n^{-1/2} +m^{-1/2}\right)$.

We evaluate the performance of \MREC algorithm in two different machine learning tasks (with convex landscapes) and compare with the existing methods in \citep{zhang2012communication}.  
We show via experiments, for the $n=1$ regime, that \MRE algorithm outperforms these algorithms. The observations are also in line with the expected error bounds we give in this paper and those previously available. In particular, in the $n=1$ regime, the expected error  of \MRE 
algorithm goes to zero as the number of machines increases, while the expected errors of the previously available estimators remain lower bounded by a  constant.

Unlike existing works, our results concern a regime where the number of machines $m$ is large, and our bounds tend to zero as $m$ goes to infinity, even if the number of per-machine observations ($n$) is bounded by a constant.
This is contrary to the algorithms in \citep{zhang2012communication}, whose errors tend to zero only when $n$ goes to infinity. 
In fact, when $n=1$, a simple example\footnote{Consider two convex functions $ f_0(\theta)=\theta^2+\theta^3/6$ and 
	$ f_1(\theta)=(\theta-1)^2+(\theta-1)^3/6$ over $[0,1]$. 
	Consider a  distribution $P$ that associates probability $1/2$ to each function.
	Then, $\mathbb{E}_P[f(\theta)]=f_0(\theta)/2+f_1(\theta)/2$, and the optimal solution is $\theta^*=(\sqrt{15}-3)/2\approx 0.436$. 
	On the other hand, in the averaging method proposed in \citep{zhang2012communication}, assuming $n=1$, 
	the empirical minimizer of each machine is either $0$ if it observes $f_0$, or $1$ if it observes $f_1$.
	Therefore, the server receives messages $0$ and $1$  with equal probability , and $\E\big[\hat{\theta}\big]=1/2$.
	Hence, $\mathbb{E}\big[|\hat{\theta}-\theta^*|\big]>0.06$,  for all values of $m$.}
shows that the expected errors of the simple Averaging and Bootstrap algorithms in \citep{zhang2012communication} remain lower bounded by a constant, for all values of $m$. The algorithm in \citep{jordan2018communication} suffers from  a similar problem and its expected error may not go to zero  when $n=1$.

\subsection{Outline}
The paper is organized as follows. 
We begin with a detailed model and problem definition in Section~\ref{sec:problem def}.
We then propose  our lower bound on the estimation error in Section~\ref{sec:main lower convex}, under general communication constraints.
In Section~\ref{sec:main upper convex}, we present the \MREC algorithm and its error upper bound.
Section~\ref{sec:constant} then provides our results for the regime where  communication budget is limited to constant number of bits per transmission.
After that, we report our numerical experiments in Section~\ref{sec:numerical}.
Finally, in Section~\ref{sec:discussion} we conclude the paper and discuss several open problems and directions for future  research. 
All proofs are relegated to the appendix for improved readability.


\section{Problem Definition} \label{sec:problem def}
Consider a positive integer $d$ and a  collection  $\mathcal{F}$ of real-valued convex functions over $[-1,1]^d$. 
Let $P$ be an unknown probability distribution over the functions in $\mathcal{F}$.  
Consider the expected loss function
\begin{equation} \label{eq:def of the loss F}
F(\theta) = \mathbb{E}_{f\sim P}\big[f(\theta)\big], \qquad \theta\in [-1,1]^d.
\end{equation} 
Our goal is to learn a parameter $\theta^*$ that minimizes $F$:
\begin{equation}\label{eq:opt problem}
\theta^*=\argmin{\theta \in [-1,1]^d}\, F(\theta) .
\end{equation}
The expected loss is to be minimized in a distributed fashion, as follows.
We consider a distributed system comprising $m$ identical  machines and a server. 
Each machine $i$ has access to a set of $n$ independently and identically distributed samples $\{f^i_1,\cdots,f^i_n\}$ drawn from the probability distribution $P$. 
Based on these observed functions,   machine $i$ then sends a signal $Y^i$ to the server.  
We assume that the length of each signal is limited to $b$ bits. 
The server then collects signals  $Y^1,\ldots, Y^m$ and outputs an estimation of $\theta^*$, which we denote by $\hat{\theta}$. 
See Fig.~\ref{Fig:system} for an illustration of the  system model.\footnote{The considered model here is similar to the one in \citep{saleh2019}.} 

We let the following assumptions be in effect throughout the paper:
\begin{assumption}[Differentiability]\label{ass:1}
	We assume:
	\begin{itemize}
		\item  Each $f\in\F$  is once differentiable  and its derivatives are bounded and Lipschitz continuous. More concretely, for any $f\in\F$ and any $\theta,\theta'\in [-1,1]^d$, we have $|f(\theta)|\le \sqrt{d}$,  $\|\nabla f(\theta)\|\leq 1$, and $\|\nabla f(\theta) - \nabla f(\theta')\|\leq \|\theta-\theta'\|$.
		\item The minimizer of $F$ lies in the interior of the cube $[-1,1]^d$. Equivalently, there exists $\theta^*\in (-1,1)^d$  such that $\nabla F(\theta^*) = \mathbf{0} $.  
	\end{itemize}
\end{assumption}

In Assumption~\ref{ass:1} we consider a class of functions with Lipschitz continuous first order derivatives, compared to previous works that consider function classes with Lipschitz continuous second or third order derivatives~\citep{zhang2013information, jordan2018communication}. This broadens the scope and applicability of our model to learning tasks where the loss landscape is far from being smooth (see Section~\ref{sec:discussion} for further discussions).

\begin{assumption}[Convexity]\label{ass:2}
	We assume:
	\begin{itemize}
		\item Every $f\in \F$	 is convex.
		\item  Distribution $P$ is such that $F$ (defined in \eqref{eq:def of the loss F}) is strongly convex. More specifically, there is a constant $\cc>0$ such that for any   $\theta_1,\theta_2 \in [-1,1]^d$, we have $F(\theta_2) \ge F(\theta_1) + \nabla F(\theta_1)^T (\theta_2-\theta_1) + \cc \|\theta_2-\theta_1\|^2$. 
	\end{itemize}
\end{assumption}

The convexity assumption (Assumption~\ref{ass:2}) is common in the literature of distributed learning~\citep{zhang2013information,jordan2018communication}. When $F$ is strongly convex, the objective is often designing estimators that minimize $ \Exp\big[\|\hat\theta-\theta^*\|^2\big]$. Given the upper and lower bounds on the second derivative (in Assumptions~\ref{ass:1} and~\ref{ass:2}), this is equivalent  (up to multiplicative constants) with minimization of $\Exp\big[F(\hat\theta)-F(\theta^*)\big]$. 
Note also that the assumption $\|\nabla F(x)\| \le 1$ (in Assumption~\ref{ass:2}) implies that 
\begin{equation} \label{eq:lam lower}
\lambda\le \frac1{\sqrt{d}}.
\end{equation}
This is because  if $\lambda> 1/\sqrt{d} $, then $\|\nabla F(x)\| > 1$, for some $x\in [-1,1]^d$.



\medskip
\section{Main Lower Bound}
\label{sec:main lower convex}
The following theorem shows that in a system with $m$ machines, $n$ samples per machine, and $B$ bits per signal transmission, no estimator can achieve estimation error less than $\|\hat\theta-\theta^*\|=\tilde{\Omega}\left(\max\left(n^{-1/2} (mB)^{-1/d},\,(mn)^{-1/2} \right)\right)$. 
The proof is given in Appendix~\ref{app:proof th lower c}.

\begin{theorem}
	\label{th:lower c}
	Suppose that Assumption~\ref{ass:2} is in effect for $\lambda\le 1/\big(10\sqrt{d} \big)$. 
	Then, for any estimator with output $\hat{\theta}$, there exists a probability distribution over $\F$ such that 
	\begin{equation} \label{eq:main lower bound}
	\Pr\left(\|\hat{\theta}-\theta^*\|=\tilde{\Omega}\left(\max\left(\frac{1}{\sqrt{n}\, (mB)^{1/d}},\, \frac{1}{\sqrt{mn}} \right)\right)^{\phantom{1\!}}\right)\geq \frac{1}{3}.
	\end{equation}
	More specifically, for large enough values of $mn$, for any estimator there is a probability distribution over $\F$  such that with probability at least $1/3$, 
	\begin{equation} \label{eq:main lower bound 2}
	\|\hat{\theta}-\theta^*\|\,\ge\,\max\left(\frac1{640\times50^{1/d}\,d^{2.5}\,\log^{2+3/d}(mn)} \times\frac{1}{\sqrt{n}\,(mB)^{1/d}}\,,\, \frac{\sqrt{d}}{5\sqrt{mn}} \right) 
	\end{equation}
\end{theorem}

In light of \eqref{eq:lam lower}, the assumption $\lambda\le 1/\big(10\sqrt{d} \big)$ in the statement of the theorem appears to be innocuous, and is merely aimed to facilitate the proofs.
The proof is given in Section~\ref{app:proof th lower c}.
The key idea is to show that finding an $O\big(n^{-1/2} m^{-1/d})$-accurate minimizer of $F$ (i.e.,  $\|\hat\theta-\theta^*\|=O\big(n^{-1/2} m^{-1/d})$) is as difficult as finding an $O\big( n^{-1/2} m^{-1/d})$-accurate approximation of $\nabla F$ for all points in an $n^{-1/2}$-neighborhood of $\theta^*$. 
This is quite counter-intuitive, because the latter problem looks way more difficult than the former. 
To see the unexpectedness  more clearly, it suggests that in the  special case where $n=1$,  finding an $m^{-1/d}$-approximation of $\nabla F$ over the entire domain is no harder than finding an $m^{-1/d}$-approximation of $\nabla F$ at a  single (albeit unknown) point  $\theta^*$. 
This provides a key insight beneficial for devising estimation algorithms:
\begin{insight} \label{ins:1}
Finding an $\tilde{O}\big(n^{-1/2} m^{-1/d})$-accurate minimizer of $F$ is  as difficult as finding an $O\big( n^{-1/2} m^{-1/d})$-accurate approximation of $\nabla F$ over an $n^{-1/2}$-neighborhood of $\theta^*$.
\end{insight}
This inspires estimators that first approximate $\nabla F$ over a neighborhood of $\theta^*$ and then choose  $\hat\theta$ to be a point with minimum $\|\nabla F\|$. 
We follow a similar idea in Section~\ref{sec:main upper convex} to design the \MREC algorithm with  order optimal error.

As an immediate corollary of Theorem~\ref{th:main upper c}, we obtain a lower bound on the moments of estimation error.
\begin{corollary}
	For any estimator $\hat{\theta}$, there exists a probability distribution over $\F$ such that for any $k\in \mathbb{N}$, 
	\begin{equation}\label{eq:lower moments}
	\mathbb{E}\left[\|\hat{\theta}-\theta^*\|^k\right]=\tilde{\Omega}\left(\max\left(\frac{1}{\sqrt{n}\, (mB)^{1/d} }, \, \frac{1}{\sqrt{mn}}\right)^k\right).
	\end{equation}
\end{corollary}

In view of \eqref{eq:lower moments},  no estimator can achieve performance of a centralized solution with the budget of $B=O(\log mn)$ when $d\ge3$.
As discussed earlier in the Introduction section, this is in contrast to the result in \citep{zhang2012communication} that a simple averaging algorithm achieves $O(1/\sqrt{nm})$ accuracy 
(similar to a centralized solution), 
in a regime that $n>m$. 
This apparent contradiction is resolved by the difference in the set of functions considered in the two works.
The set of functions in \citep{zhang2012communication} are twice differentiable with Lipschitz continuous second derivatives, while we do not assume existence or Lipschitz continuity of second derivatives.


\medskip
\section{\MREC Algorithm and its Error Upper Bound} \label{sec:main upper convex}
Here, we  propose an order optimal estimator under general communication budget  $B$, for $B\ge d\log mn$.
The high level idea, in view of Insight~\ref{ins:1}, is to acquire an approximation of derivatives of $F$ over a neighborhood of $\theta^*$, and then letting $\hat\theta$ be the minimizer of size of these approximate gradients.
For efficient gradient approximation, transmitted signals are designed such that the server can construct a multi-resolution view of gradient of function $F(\theta)$ around a promising grid point. Thus, we call the proposed algorithm ``Multi-Resolution Estimator for Convex loss (\MREC\!)". The description of \MREC is as follows:

Each machine $i$ observes $n$ functions and sends a signal $Y^i$ comprising $\lceil B/(d\log mn)\rceil$ sub-signals of length $\lfloor d \log mn\rfloor$. 
Each sub-signal has three parts of the form $(s,p,\Delta)$. 
The three parts $s$, $p$, and $\Delta$ are as follows.
\begin{itemize}
	\item Part $s$: Consider a grid $G$ with resolution $\log(mn)/\sqrt{n}$ over the $d$-dimensional cube $[-1,1]^d$. Each machine $i$ computes the minimizer of the average of its first $n/2$ observed functions,
	\begin{equation} \label{eq:def theta i half upper c}
	\theta^i=\argmin{\theta \in [-1,1]^d} \,\sum_{j=1}^{n/2} f_j^i(\theta).
	\end{equation}
	It then lets 	$s$ be the closest grid point to $\theta^i$. 
	Note that all sub-signals of a machine have the same $s$-part.
	\item Part $p$: Let
	\begin{equation}\label{eq:def delta c}
	\delta \,\triangleq \,  2d \log^{3}(mn) \, \max \left(\frac1{(mB)^{1/{d}}}, \, \frac{2^{d/2}}{m^{1/2}} \right).
	\end{equation}
	Let $t = \log(1/\delta)$. 
	Without loss of generality we assume that $t$ is a non-negative integer.\footnote{If $\delta>1$, we reset the value of $\delta$ to $\delta=1$. It is not difficult to check that the rest of the proof would not be upset in this spacial case.}
	Let $C_s$  be a $d$-dimensional cube with edge size $2\log(mn)/\sqrt{n}$ centered at $s$.
	Consider a sequence of $t+1$ grids on $C_s$ as follows.
	For each $l=0,\ldots,t$, we partition the cube $C_s$ into $2^{ld}$ smaller equal sub-cubes with edge size $2^{-l+1} \log(mn)/\sqrt{n}$. 
	The $l$th grid $\tilde{G}_s^l$ comprises the centers of these smaller cubes.
	Then, each $\tilde{G}_s^l$ has $2^{ld}$ grid points.
	For any point  $p'$ in $\tilde{G}_s^l$, we say that $p'$ is the parent of all $2^d$ points in $\tilde{G}_s^{l+1}$
	that are in the $\big(2^{-l}\times (2\log mn)/\sqrt{n}\big)$-cube centered at $p'$ (see Fig. \ref{Fig:MRE}). 
	Thus, each point $\tilde{G}_s^l$ ($l<t$) has $2^d$ children. 
	
	In each sub-signal, to select $p$, we randomly choose an $l$ from $0,\dots, t$ with probability
	\begin{equation}\label{eq:prob choose p c}
	\Pr(l) = \frac{2^{(d-2)l}}{\sum_{j=0}^t 2^{(d-2)j}}
	\end{equation}
	We then let $p$ be a uniformly chosen random grid point in $\tilde{G}_s^l$.
	The level $l$ and point $p$ chosen in different sub-signals of a machine are independent and have the same distribution. 
	Note that $O(d\log(1/\delta))=O(d\log(mn))$ bits suffice to identify $p$ uniquely. 
	
	\item Part $\Delta$: We let 
	\begin{equation}\label{eq:def hat F upper}
	\hat{F}^i(\theta)\triangleq \frac2n\sum_{j=n/2+1}^{n} f_j^i(\theta),
	\end{equation}
	and refer to it as the empirical function of the $i$th machine.
	For each sub-signal, if the selected $p$ in the previous part is in $\tilde{G}_s^0$, i.e., $p=s$, then we set $\Delta$ to the gradient of $\hat{F}^i$  at $\theta=s$.
	Otherwise, if $p$ is in $\tilde{G}_s^l$ for $l\geq 1$, we let
	\begin{equation*}
	\Delta \,\triangleq\, \nabla \hat{F}^i (p)-\nabla \hat{F}^i(p'),
	\end{equation*}
	where $p'\in \tilde{G}_s^{l-1}$ is the parent of $p$. 
	Note that $\Delta$ is a $d$-dimensional vector whose entries range over  
	$\big(2^{-l}\sqrt{d}\log(mn)/\sqrt{n}\big) \times \big[-1,+1\big]$.
	This is due to the Lipschitz continuity of the derivative of the functions in $\F$ (see Assumption~\ref{ass:1})  and the fact that $\|p-p'\|=2^{-l}\sqrt{d}\log(mn)/\sqrt{n}$. 
	Hence,  $O(d\log(mn))$  bits suffice to represent $\Delta$ within accuracy $2\delta\log(mn)/\sqrt{n}$. 
\end{itemize}

\begin{figure}[t]
	\centering
	\includegraphics[width=8cm]{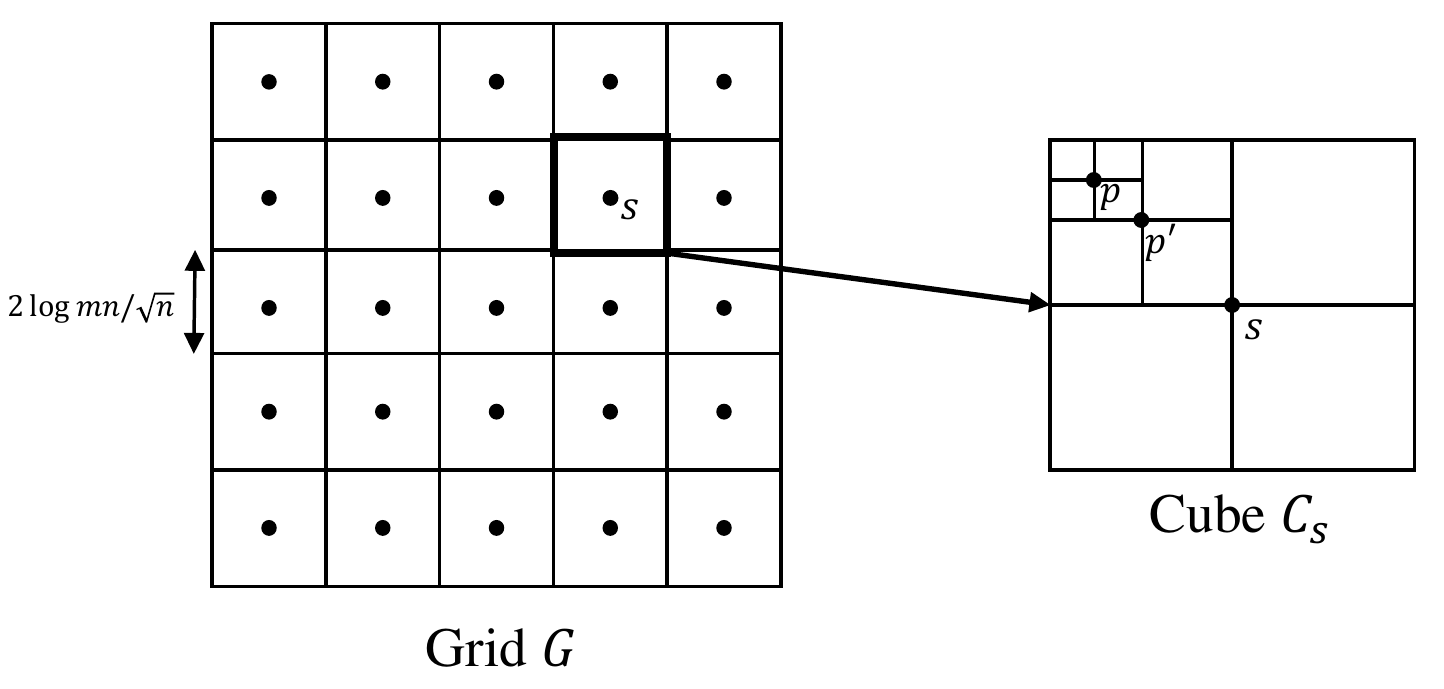}
	\caption{An illustration of grid $G$ and cube $C_s$ centered at point $s$ for $d=2$. The point $p$ belongs to $\tilde{G}_s^2$ and $p'$ is the parent of $p$.}
	\label{Fig:MRE}
\end{figure}

At the server, we choose an $s^*\in G$ that has the largest number of occurrences in the received signals. 
Then, based on the signals corresponding to $\tilde{G}_{s^*}^0$,
we approximate the gradients of $F$ over $C_{s^*}$ as follows.
We first eliminate redundant sub-signals so that no two surviving sub-signals from a same machine have the same $p$-parts (consequently, for each machine, the surviving sub-signals are distinct). We call this process ``redundancy elimination''. 
We then let $N_{s^*}$ be the total number of surviving sub-signals that contain $s^*$ in their  $p$ part, and compute
\begin{equation*}
\hat{\nabla} F(s^*)=\frac{1}{N_{s^*}}\sum_{\substack{\textrm{\scriptsize{Subsignals of the form }} \\ (s^*,s^*,\Delta)\\ \mbox{\scriptsize{after redundancy elimination}}}}\Delta,
\end{equation*}	
Then, for any point $p\in \tilde{G}_{s^*}^l$ with $l\geq 1$, we let
\begin{equation}
\hat{\nabla}F(p)=\hat{\nabla}F(p')+\frac{1}{N_p}\sum_{\substack{\mbox{\scriptsize{Subsignals of the form }} \\ (s^*,p,\Delta)\\ \mbox{\scriptsize{after redundancy elimination}}}}\Delta,
\label{eq:page17 c}
\end{equation}
where $N_p$ is the number of signals having point $p$ in their second argument, after redundancy elimination.
Finally, the sever lets $\hat{\theta}$ be a grid point $p$ in $\tilde{G}_{s^*}^t$  with the smallest $\|\hat{\nabla}F(p)\|$.


\begin{theorem}\label{th:main upper c}
	Let $\hat{\theta}$ be the output of the above algorithm. Then, 
\begin{equation*}
\Pr\left(\|\hat{\theta}-\theta^*\|> \frac{{4d^{1.5} \log^{4}(mn)}}{\lambda} \, \max \left(\frac1{\sqrt{n}\,(mB)^{1/{d}}}, \, \frac{2^{d/2}}{\sqrt{mn}} \right)\right) \,=\, \exp\Big(-\Omega\big(\log^2(mn)\big)\Big).
\end{equation*}
\end{theorem}
The proof is given in  Appendix~\ref{sec:proof main alg c}.
The proof goes by first showing that $s^*$ is a closest grid point of $G$ to $\theta^*$ with high probability. 
We then show that for any $l\le t$ and any $p\in \tilde{G}_{s^*}^l$, the number of received signals corresponding to $p$ is large enough so that the server obtains a good approximation of $\nabla F$ at $p$. 
Once we have a good approximation $\hat\nabla F$ of $\nabla F$ at all points of $\tilde{G}_{s^*}^t$, a point at which $\hat\nabla F$ has the minimum norm lies close to the minimizer of $F$.

\begin{corollary} \label{cor:upper c}
	Let $\hat{\theta}$ be the output of the above algorithm. 
	There is a constant $\eta>0$ such that 
	for any $k\in \mathbb{N}$,
	\begin{equation*}
	\mathbb{E}\big[\|\hat{\theta}-\theta^*\|^k\big]\,<\,\eta \,\left(\frac{{4d^{1.5} \log^{4}(mn)}}{\lambda} \, \max \left(\frac1{\sqrt{n}\,(mB)^{1/{d}}}, \, \frac{2^{d/2}}{\sqrt{mn}} \right)\right)^k.
	\end{equation*}
	Moreover, $\eta$ can be chosen arbitrarily close to $1$, for large enough values of $mn$.
\end{corollary}


The upper bound in Theorem~\ref{th:main upper c} matches the lower bound in Theorem~\ref{th:lower c} up to a polylogarithmic factor. 
In this view, the \MREC algorithm has order optimal error.
Moreover, as we show in Appendix~\ref{sec:proof main alg c}, in the course of computations, the server  obtains  an approximation $\hat{F}$  of $F$ such that for any $\theta$ in the cube $C_{s^*}$, we have $\|\nabla \hat{F}(\theta) -\nabla F(\theta)\| = \tilde{O}\big( m^{-1/d}n^{-1/2})$. 
Therefore, the server not only finds the minimizer of $F$, but also obtains an approximation of $F$ at all points inside $C_{s^*}$. 
This is in line with our previous observation in Insight~\ref{ins:1}.


\medskip
\section{Learning under Tiny Communication Budget} \label{sec:constant}
In this section, we consider the regime that  communication budget per transmission is bounded by a constant, i.e., $B$ is a constant independent of $m$ an $n$.  We present a  lower bound on the estimation error and  propose an estimator whose error vanishes as $m$ and $n$ tend to infinity.

We begin with a lower bound. The next theorem shows that when $n=1$, the expected error is lower bounded by a constant, even if $m$ goes to infinity.  
\begin{theorem} \label{th:Hinf}
	Let $n=1$ and suppose that the  number of bits per signal, $B$, is limited to a constant.
	Then, there is a distribution $P$ over $\F$ such that expected error, $\mathbb{E}_{P}\left[\|\hat{\theta}-\theta^*\|\right]$, of any randomized estimator $\hat{\theta}$ is lower bounded by a constant, 
	for all $m\geq 1$.
	The constant lower bound holds even when $d=1$.
\end{theorem}	
The proof is given in Appendix~\ref{sec:proof constant bit lower bound}. 
There, we construct a distribution $P$ that associates non-zero probabilities to $2^b+2$ polynomials of order at most $2^b+2$. 
Theorem~\ref{th:Hinf} shows that the expected error is bounded from below by a constant regardless of $m$, when $n=1$ and $B$ is a constant.

We now show that the expected error can be made arbitrarily small as $m$ and $n$ go to infinity simultaneously.
\begin{theorem}
	Under the communication budget of $B=d$ bits per transmission, 
	there exists a randomized estimator $\hat{\theta}$ such that
	\begin{equation*}
	\mathbb{E}\left[\|\hat{\theta}-\theta^*\|^2\right]^{1/2}=O\left(\frac{1}{\sqrt{n}}+\frac{1}{\sqrt{m}}\right).
	\end{equation*}
	\label{th:H2constb}
\end{theorem}
The proof is  given in Appendix~\ref{app:constant bit upper bound}. 
There, we propose a simple randomized algorithm in which each machine $i$ first computes an $O(1/\sqrt{n})$-accurate estimation $\theta^i$ based on its observed functions. It then generates as its output signal a random binary sequence of length $d$ whose $j$th entry is $1$ with probability $(1+\theta_j^i)/2$, where $\theta_j^i$ is the $j$th entry of $\theta^i$.
The server then computes $\hat\theta$ based on the average of received signals.





\section{Experiments} \label{sec:numerical}
\vspace{-.1cm}
We evaluated the performance of \MREC on two learning tasks and compared with the averaging method (AVGM) in \citep{zhang2012communication}. Recall that in AVGM, each machine sends the empirical risk minimizer of its own data to the server and the average of received parameters at the server is returned in the output.

 The first experiment concerns the problem of ridge regression. Here, each sample $(X,Y)$ is generated based on a linear model $Y=X^T\theta^*+E$, where $X$, $E$, and $\theta^*$ are  sampled from $N(\mathbf{0},I_{d\times d})$, $N(0,0.01)$, and uniform distribution over $[0,1]^d$, respectively. We consider square loss function with $l_2$ norm regularization: $f(\theta)=(\theta^TX-Y)^2+0.1 \|\theta\|_2^2$.
In the second experiment, we perform a logistic regression task, considering sample vector $X$ generated according to $N(\mathbf{0},I_{d\times d})$ and labels $Y$ randomly drawn from $\{-1,1\}$ with probability $\Pr(Y=1|X,\theta^*)=1/(1+\exp(-X^T \theta^*))$.
In both experiments, we consider a two dimensional domain ($d=2$) and assumed that each machine has access to one sample ($n=1$).

\begin{figure}[t!]
	\centering
	\begin{subfigure}[t]{0.5\textwidth}
		\centering
		\includegraphics[width=2.35in]{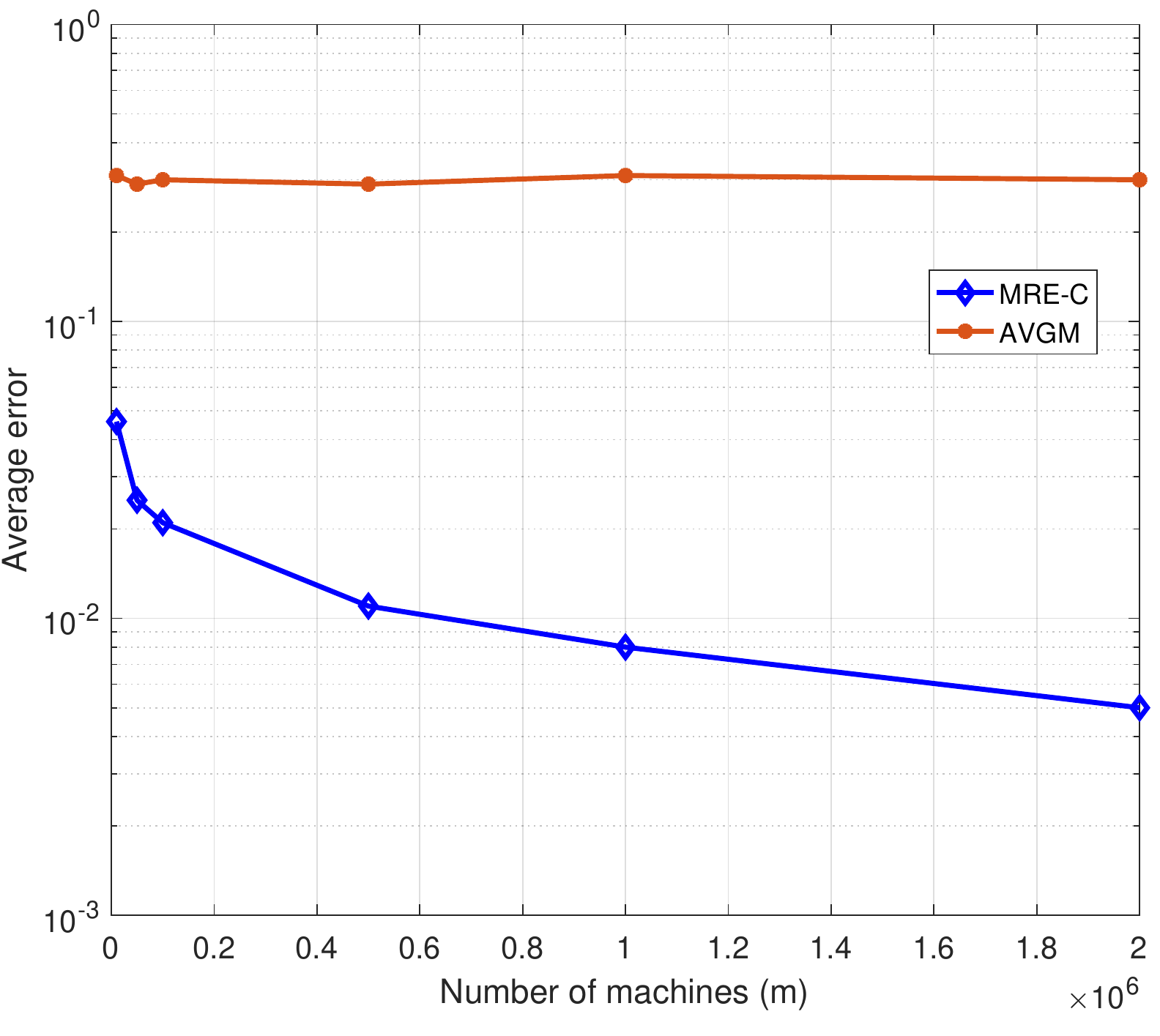}
		\caption{Ridge regression}
	\end{subfigure}%
	~ 
	\begin{subfigure}[t]{0.5\textwidth}
		\centering
		\includegraphics[width=2.37in]{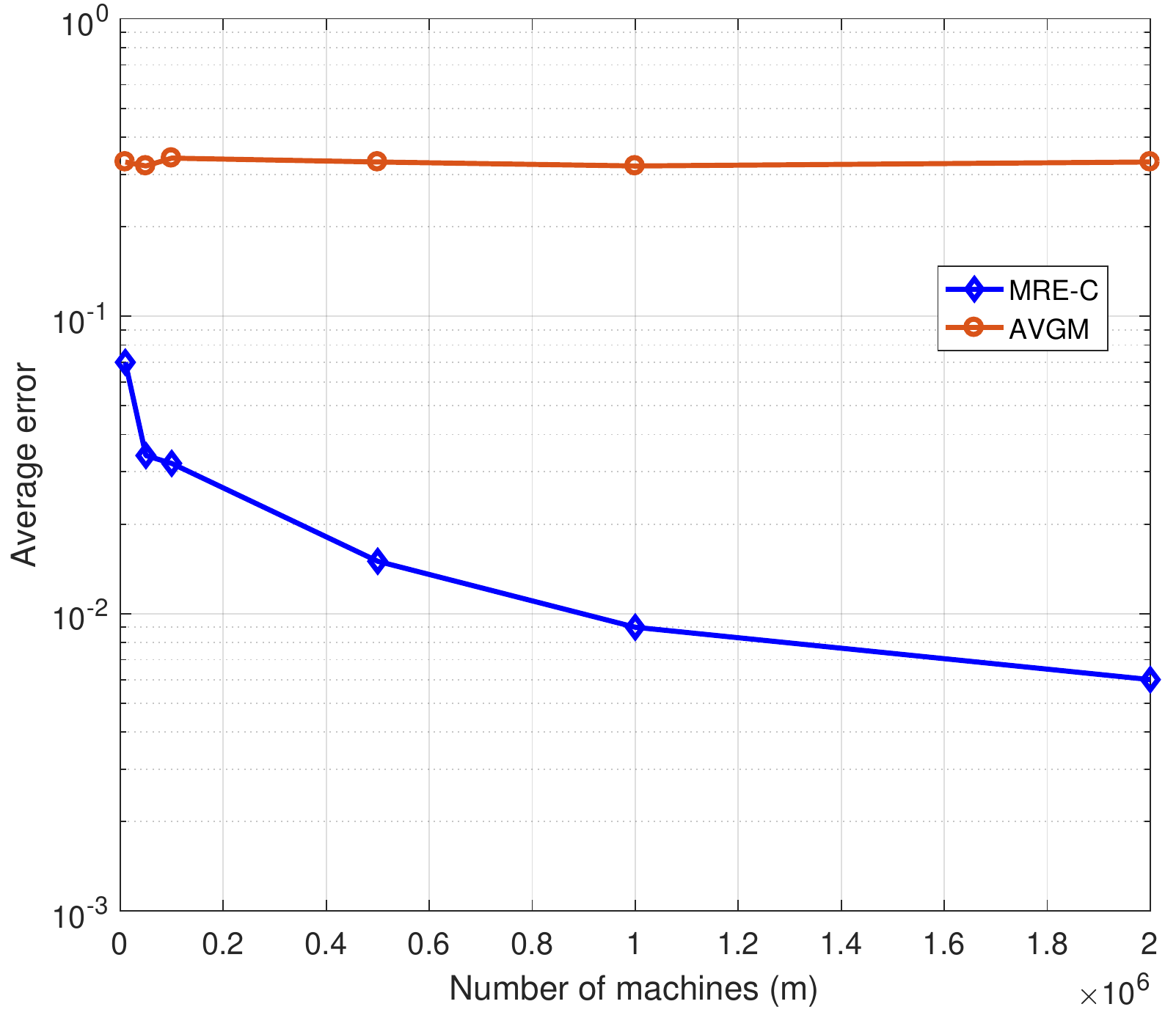}
		\caption{Logistic regression}
	\end{subfigure}
	\caption{The average  of \MREC and AVGM algorithms versus the number of machines in two different learning tasks. }
	\label{fig:sim}
\end{figure}

In Fig. \ref{fig:sim}, the average of $\|\hat{\theta}-\theta^*\|_2$ is computed over $100$ instances  for the different number of machines in the range $[10^4,10^6]$. Both experiments suggest that the average error of \MREC keep decreasing as the number of machines increases.  This is  consistent with  the result in Theorem~\ref{th:main upper c}, according to which  the expected error of \MREC is upper bounded by  $\tilde{O}(1/\sqrt{mn})$. 
It is evident from the error curves that \MREC outperforms the AVGM algorithm in both tasks. This is  because where $m$ is much larger than $n$, the expected error of the  AVGM algorithm typically scales as $O(1/n)$, independent of $m$.


\medskip
\section{Discussion} \label{sec:discussion}
We studied the problem of statistical optimization of convex loss landscapes in a distributed system with one-shot communications. 
We present matching upper and lower bounds on estimation error under general communication constraints.
We showed that the expected error of any estimator is lower bounded by $\|\hat\theta-\theta^*\|=\tilde{\Omega}\left(\max\left(n^{-1/2} (mB)^{-1/d},\,(mn)^{-1/2} \right)\right)$. 
We proposed an algorithm called \MREC\!, whose estimation errors meet the above lower bound up to a poly-logarithmic factor. 
More specifically, the \MREC algorithm has error no large than $\|\hat\theta-\theta^*\|=\tilde{O}\left(\max\left(n^{-1/2} (mB)^{-1/d},\,(mn)^{-1/2} \right)\right)$. 
Aside from being order optimal, the \MREC algorithm has the advantage over the existing estimators that its error tends to zero as the number of machines goes to infinity, even when the number of samples per machine is upper bounded by a constant and  communication budget is limited to $d\log mn$ bits per transmission.
This property is in line with the out-performance of the \MREC algorithm in the  $m\gg n$ regime, as verified in our experimental results. 

The key insight behind the proof of the lower bound and the design of our algorithm is an observation that emerges from our proofs: in the one-shot model, finding an $O\big(n^{-1/2} m^{-1/d})$-accurate minimizer of $F$  is as difficult as finding an $O\big( n^{-1/2} m^{-1/d})$-accurate approximation of $\nabla F$ for all points in an $n^{-1/2}$-neighborhood of $\theta^*$. 
Capitalizing on this observation, the \MREC algorithm  computes, in an efficient way, an approximation of the gradient of the expected loss  over a neighborhood of $\theta^*$. 
It then output a minimizer of approximate gradient norms as its estimate of the loss minimizer. 
It is quite counter intuitive that while \MREC algorithm carries out an intricate and  seemingly redundant task of approximating the loss function for all points in a region, it is still very efficient, and indeed order optimal in terms of estimation error and sample complexity. 
This remarkable observation is in line with the above insight that finding an approximate minimizer is as hard as finding an approximation of the function over a relatively large neighborhood of the minimizer.

We also addressed the problem of distributed learning under tiny (constant) communication budget. 
We showed that when budget $B$ is a constant and for $n=1$, the expected error of any estimator is lower bounded by a constant, even when $m$ goes to infinity. 
We then proposed an estimator with the budget of $B=d$ bits pert transmission and showed that its expected error is no larger than $O\left(n^{-1/2}+m^{-1/2}\right)$.

Our algorithms and bounds are designed and derived for a broader class of functions with Lipschitz continuous first order derivatives, compared to the previous works that consider function classes with Lipschitz continuous second or third order derivatives. 
The assumption is indeed both  practically important  and technically challenging. 
For example, it is well-known that the loss landscapes involved in learning applications and neural networks are highly non-smooth. Therefore, relaxing assumptions on higher order derivatives is actually a practically important improvement over the previous works.
On the other hand, considering Lipschitzness only  for the first order derivative  renders the problem way more difficult. To see this, note that when $n>m$, the existing upper bound $O\left((mn)^{-1/2}+n^{-1}\right)$ for the case of Lipschitz second derivatives  goes below the  $O(m^{-1/d}n^{-1/2})$ lower bound in the  case of Lipschitz first derivatives.

A drawback of the \MRE algorithms is that each machine requires to know $m$ in order to set the number of levels for the grids. This however can be resolved by considering infinite number of levels, and letting the probability that $p$ is chosen from level $l$ decrease exponentially with $l$. 
The constant lower bound in Theorem~\ref{th:Hinf} decreases exponentially with $B$. This we expect because when $B=d\log mn$,  error of the \MREC algorithm is proportional to an inverse polynomial of $m$ and $n$ (see Theorem~\ref{th:main upper c}, and therefore decays exponentially with $B$.


There are several open problems and directions for future research. 
The first group of problems involve the constant bit regime. 
It would be interesting if one could verify whether or not the bound in Theorem~\ref{th:H2constb} is order optimal. 
We conjecture that this bound is tight, and no estimator has expected error smaller than $o\left(n^{-1/2}+m^{-1/2}\right)$, when the communication budget is bounded by a constant. This would essentially be  an extension of Theorem~\ref{th:Hinf} for $n>1$.
Another interesting problem involves the regime $B<d$, and best accuracy achievable with $B<d$ bits per transmission?

As for the \MREC estimator, the estimation error of these algorithms are  optimal up to poly-logarithmic factors in $m$ and $n$. However, the bounds in Theorem~\ref{th:main upper c}  have  an extra exponential dependency on $d$. Removing this exponential dependency is an important problem to address in future works.

More importantly, an interesting problem involves the relaxation of the convexity assumption (Assumption~\ref{ass:2}) and finding tight lower bounds and order-optimal estimators for general non-convex loss landscapes, in the  one-shot setting. This we address in an upcoming publication (see \cite{AlSG20nc} for first drafts).

Another important group of problems concerns a more restricted class of functions with Lipschitz continuous second order derivatives. 
Despite several attempts in the literature, the optimal scaling of expected error for this class of functions in the $m\gg n$ regime is still an open problem.




\section*{Acknowledgments}
This research was supported by INSF under contract No. 97012846.
The first author thanks Nitin Vaidya for giving invaluable insights about the problem of distributed learning. The second author thanks John N. Tsitsiklis and Philippe Rigollet for fruitful discussions  on Lemma~\ref{lemma:metric}.

\bibliography{ref}
\bibliographystyle{plainnat}

\newpage

\medskip
\medskip
\section*{\centering \huge Appendices }
\appendix

\section{Preliminaries} \label{app:prelim}
In this appendix, we review some preliminaries that will be used in the proofs of our main results.

\subsection{Concentration inequalities}
We collect two well-known concentration inequalities in the following lemma.
	\begin{lemma} (Concentration inequalities)
		\begin{enumerate}
			\item[(a)] (Hoeffding's inequality)  
			Let $X_1,\cdots,X_n$ be independent random variables ranging over the interval $[a,a+\gamma]$. Let $\bar{X}=\sum_{i=1}^n X_i/n$ and $\mu =\mathbb{E}[\bar{X}]$.
			Then, for any $\alpha>0$,
			\begin{equation*}
			\Pr\big(|\bar{X}-\mu|>\alpha\big)\leq 2\exp\left(\frac{-2n\alpha^2}{\gamma^2}\right).
			\end{equation*}
			
			\item[(b)] (Theorem 4.2 in \cite{motwani1995randomized})  
			Let $X_1,\cdots,X_n$ be independent Bernoulli  random variables, ${X}=\sum_{i=1}^n X_i$, and $\mu =\mathbb{E}[{X}]$.
			Then, for any $\alpha\in(0,1]$,
			\begin{equation*}
			\Pr\big(X<(1-\alpha)\mu\big)\leq \exp\left(-\frac{\mu\alpha^2}{2}\right).
			\end{equation*}
		\end{enumerate}
		\label{lemma:CI}
	\end{lemma}

\subsection{Binary hypothesis testing}
	Let  $P_1$ and $P_2$ be probability distributions over $\{0,1\}$, such that $P_1(0)=1/2-1/5\sqrt{mn}$ and $P_2(0)=1/2+1/5\sqrt{mn}$.
	The following lemma provides a lower bound on the error probability of binary hypothesis tests between $P_1$ and $P_2$. By the error probability of a test, we mean the maximum between the probabilities of type one and type two errors. 
	\begin{lemma}\label{lem:hypo test}
		Consider probability distributions $P_1$ and $P_2$ as above.
		Then, for any binary hypothesis test $\cal{T}$ between $P_1$ and $P_2$ with $mn$ samples,  the error probability $\cal{T}$ is at least $1/3$.
	\end{lemma}
	\begin{proof}
		Let $P_1^{mn}$ and $P_2^{mn}$ be the product-distributions of $mn$ samples.
		Then, $\E\big[ P_1^{mn}  \big] = mn\left(1/2-1/5\sqrt{mn}\right)$, $\E\big[ P_2^{mn}  \big] = mn\left(1/2+1/5\sqrt{mn}\right)$, and $\var\big( P_1^{mn}  \big)= \var\big( P_2^{mn}  \big) = mn/4-1/25$.
		For large values of $mn$, the central limit theorem implies that the CDFs of $P_1^{mn}$ and $P_2^{mn}$ converge to the CDFs of $N_1 \triangleq N\big(mn(1/2-1/\sqrt{mn}),\, mn/4-1/25\big)$ and $N_2\triangleq N\big(mn(1/2+1/\sqrt{mn}),\, mn/4-1/25\big)$, respectively.
		Therefore, for large enough values of $mn$,
		\begin{equation}
		\begin{split}
		d_{TV}\big(P_1^{mn}, P_2^{mn}\big)\,&\le \, d_{TV}\big(N_1,N_2\big) \,+\,0.02\\ 
		&= \, 1-2 Q\left(\frac{1/5\sqrt{mn}}{\sqrt{1/4mn-1/\big(25(mn)^2\big)}}\right)\,+\,0.02\\
		& >  \, 1-2Q\left(0.4\right)\,+\,0.02.
		\end{split}
		\end{equation}
		where  $d_{TV}\big(P_1^{mn}, P_2^{mn}\big) $ is the total variation distance between $P_1^{mn}$ and $P_2^{mn}$,  $Q(\cdot)$ is the Q-function of normal distributions, and the equality is because the variances of  $N_1$ and $N_2$ are the same  and the two distributions have the same values at $1/5\sqrt{mn}$.
		On the other hand, it is well-known that the maximum of the two type of errors in any binary hypothesis test is lower bounded by $1/2(1-\delta)$, where $\delta$ is the  total variation distance between the underlying distributions corresponding to the hypotheses.
		Therefore, in our case, for large enough values of $mn$, the error probability of $\cal{T}$ is at least
		\begin{equation}
			\frac{1-	d_{TV}\big(P_1^{mn}, P_2^{mn}\big)}2 \,> \,\frac{2Q\left(0.4\right)\,-\,0.02}2\,>\, 0.3445-0.01\,>\frac13.
		\end{equation}
		This completes the proof of Lemma~\ref{lem:hypo test}.
	\end{proof}

\subsection{Fano's inequality}
In the rest of this appendix, we review a well-known inequality in information theory: Fano's inequality \citep{cover2012elements}. 
Consider a pair of random variables $X$ and $Y$ with certain joint probability distribution. 
Fano's inequality asserts that given an observation of $Y$ no estimator $\hat{x}$ can recover $x$ with probability of error less than $\big( H(X|Y)-1\big)/\log(|X|)$, i.e., 
\begin{equation*}
\Pr(e)\, \triangleq \,\Pr\big(\hat{x}\neq x\big) \, \geq\, \frac{H(X|Y)-1}{\log(|X|)},
\end{equation*} 
where $H(X|Y)$ is the conditional entropy and $|X|$ is the size of probability space of $X$. 
In the special case that $X$ has uniform marginal distribution, the above inequality further simplifies as follows: 
\begin{equation}
\begin{split}
H(X|Y) &= H(X,Y)-H(Y)\geq^a H(X)-H(Y)\\
&=^b\log(|X|)-H(Y)\geq \log(|X|)-\log(|Y|)\\
&\implies \Pr(e)\geq \frac{H(X|Y)-1}{\log(|X|)}\geq \frac{\log(|X|)-\log(|Y|)-1}{\log(|X|)}=1-\frac{\log(|Y|)+1}{\log(|X|)},
\end{split}
\label{eq:Fano unif}
\end{equation}  
\\($a$) Since $H(X,Y)\geq H(X)$.
\\($b$) $X$ has uniform distribution.

\section{Proof of Theorem \ref{th:lower c}} \label{app:proof th lower c}
Here, we present the proof of Theorem \ref{th:lower c}.
The high level idea is that if there is an algorithm that finds a minimizer of $F$ with high probability, then there is an algorithm that finds a fine approximation of $\nabla F$ over $O(1/\sqrt{n})$-neighborhood of $\theta^*$. 
The key steps of the proof are as follows.

We first consider a sub-collection $\S$ of $\F$ such that for any pair $f,g$ of functions in $\S$,  there is a point $\theta$ in the $O(1/\sqrt{n})$-neighborhood of $\theta^*$ such that $\|\nabla f(\theta) - \nabla g(\theta)\| \ge \epsilon/\sqrt{n}$.
We develop a metric-entropy based framework to show that such collection exists and can have as many as $\Omega(1/\epsilon^d)$  functions.
Consider a constant $\epsilon>0$ and suppose that there exists an estimator $\hat{\theta}$ that finds an $O(\epsilon/\sqrt{n})$-approximation of $\theta^*$ with high probability, for all distributions.
We generate a distribution $P$ that associates probability $1/2$ to an arbitrary function $f\in\S$, while  distributes the remaining $1/2$ probability unevenly over $2d$ linear functions.
The priory unknown probability distribution of these linear functions can displace the minimum of $F$ in an $O(1/\sqrt{n})$-neighborhood. 
Capitalizing on this observation, we show that 
the server needs to obtain an $\big(\epsilon/\sqrt{n}\big)$-approximation of $\nabla f$ all over this $O(1/\sqrt{n})$-neighborhood; because otherwise the server could mistake $f$ for another function $g\in \S$, which leads to $\Omega\big(\epsilon/\sqrt{n}\big)$-error in $\hat\theta$ for specific choices of probability distribution over the linear functions.
Therefore, the server needs to distinguish which function $f$ out of $1/\epsilon^d$ functions in $\S$ has positive probability in $P$. 
Using information theoretic tools (Fano's inequality \citep{cover2012elements}) we conclude that the total number of bits delivered to the server (i.e., $m b$ bits) must exceed the size of $\S$ (i.e., $\Omega(1/\epsilon^d)$). 
This implies that $\epsilon \ge (m b)^{1/d}$, and no estimator has error less than $O\big((m b)^{-1/d}n^{-1/2})$.

\subsection{Preliminaries}
Before going through the details of the proof, in this subsection we present some definitions and auxiliary lemmas. 	
Here, we will only consider a sub-collection of functions in $\F$ whose derivatives vanish at zero, i.e. $\nabla f(\mathbf{0})=\mathbf{0}$,
where $\mathbf{0}$ is the all-zeros vector.
Throughout the proof, we fix constant
\begin{equation} \label{eq:def c}
c \triangleq 4d\log(mn).
\end{equation}
Let $\Ft$ be a sub-collection of functions in $\F$ that are $\lambda$-strongly convex, that is for any $f\in\F_\lambda$ and any $\theta_1,\theta_2\in[0,1]^d$, we have $f(\mathbf{0}) =0$, $\nabla f(\mathbf{0})=\mathbf{0}$ and $f(\theta_2)\ge f(\theta_1)+(\theta_2-\theta_1)^T\nabla f(\theta_1)+\lambda \|\theta_2-\theta_1\|^2$. 
\begin{definition}[$(\epsilon,\delta)$-packing and metric entropy] \label{def:metric ent}
	Given $\epsilon,\delta>0$, a subset $S\subseteq \Ft$ is said to be an $(\epsilon,\delta)$-packing if 
	for any $f\in S$, $f(\mathbf{0}) =0$ and $\nabla f(\mathbf{0})=\mathbf{0}$; and  
	for any $f,g\in S$, there exists $x\in [-\delta,\delta]^d$ such that $\|\nabla f(x)-\nabla g(x)\|\geq \epsilon$. 
	We denote an $(\epsilon,\delta)$-packing with maximum size  by $S^*_{\epsilon,\delta}$, and refer to $K_{\epsilon,\delta}\triangleq \log |S^*_{\epsilon,\delta}|$ as the $(\epsilon,\delta)$-metric entropy. 
\end{definition}

\begin{lemma}
	For any $\epsilon,\delta \in (0,1)$ with $10\sqrt{d} \epsilon \leq \delta$, we have $K_{\epsilon,\delta}\geq  \left(\delta/20\epsilon\sqrt{d}\right)^d$.
	\label{lemma:metric}
\end{lemma}
The proof is given in Appendix~\ref{app:proof metric ent}.
The proof is constructive. 
There we devise a set of functions in $\Ft$ as convolutions  of a collection of impulse trains by a suitable kernel, and show that they form a packing.

Choose an arbitrary $\epsilon$ such that
\begin{equation} \label{eq:lower on eps}
\epsilon < \frac{1}{80 d^{1.5} \log^{1+3/d}(mn)} \,\left( \frac{1}{50mB}  \right)^{1/d},
\end{equation}
and consider an $\big(\epsilon/\sqrt{n},1/(c\sqrt{n})\big)$-packing with maximum size and denote it by  $S^*$ (see Definition~\ref{def:metric ent}), where $c$ is the constant in \eqref{eq:def c}. We fix this $\epsilon$ and $S^*$ for the rest of the proof.
Then, 
\begin{equation} \label{eq:lower k eps n}
K_{\epsilon/\sqrt{n},1/(c\sqrt{n})} \,\ge \, \left(\frac{1/\left(c\sqrt{n}\right)}{20\sqrt{d}\,\times\left(\epsilon/\sqrt{n}\right)}\right)^d
\,> 50mB \,\log^3(mn),
\end{equation}
where the first inequality follows from Lemma~\ref{lemma:metric} and the second inequality is by substituting the values of $c$ and $\epsilon$ from \eqref{eq:def c}  and \eqref{eq:lower on eps}.

We now define a collection $\mathcal{C}$  of probability distributions over $\Ft$.
\begin{definition}[Collection $\mathcal{C}$ of probability distributions] \label{def:class C}
	Let $\mathbf{e}_i$ be a vector whose $i$-th entry equals $1$ and all other entries equal zero.
	Consider a pair of linear functions $g^+_i(x)=d\mathbf{e}_i^Tx$ and $g^-_i(x)=-d\mathbf{e}_i^Tx$.
	Then, the collection $\mathcal{C}$  consists of probability distributions $P$ of the following form:
	\begin{equation*}
	P:\begin{cases}
	P(f)=\frac{1}{2}, &\mbox{ for some } f\in S^*\\
	P(g^+_i)\in \big[\frac{1}{4d}-\frac{1}{2c\sqrt{n}},\frac{1}{4d}+\frac{1}{2c\sqrt{n}}\big], & i=1,\ldots,d,\\
	P(g^-_i)=\frac{1}{2d}-P(g^+_i), & i=1,\ldots,d,
	\end{cases}
	\end{equation*}
	where $c$ is the constant defined in \eqref{eq:def c}. 
	For each $f \in S^*$, we refer to any such $P$ as a corresponding distribution of $f$. 
\end{definition}

Note each $f\in S^*$, corresponds to infinite number of distributions in $\mathcal{C}$. 
In order to simplify the presentation, we use the shorthand notations $P_0=P(f)=1/2$, $P_i^+=P(g_i^+)$, and $P_i^-=P(g_i^-)$, for $i=1,\ldots,d$.
The following lemma provides a bound on the distance of distinct distributions in $\C$, in a certain sense.
\begin{lemma}
	Consider a function $f\in S^*$ and two corresponding distributions $P,P'\in \mathcal{C}$. 
	We draw $n$ i.i.d. samples from $P$. 
	Let $n_0,n_{1^+},\ldots,n_{d^+},n_{1^-},\ldots, n_{d^-}$ be the number of samples that equal  $f,g^+_1,\ldots, g^+_d,g^-_1,\ldots,g^-_d$, respectively. Let $\underline{n}=[n_0,n_{1^+},n_{1^-},\ldots,n_{d^+},n_{d^-}]$. Then, 
	\begin{equation*}
	\Pr_{\underline{n}\sim P} \Bigg(\frac{1}{2}\leq \frac{P(\underline{n})}{P'(\underline{n})}\leq 2\Bigg) \,=\, 1-\exp\Big(-\Omega\big(\log^2(mn)\big)\Big).
	\end{equation*}
	\label{Lemma:ratio}
\end{lemma}
The proof is based on the Hoeffding's inequality (Lemma~\ref{lemma:CI}~(a)) and is given in Appendix~\ref{app:proof flatness}.

\medskip
\subsection{Proof of $\Omega\big((mB)^{-1/d}n^{-1/2}\big)$ bound}
Recall the definition of $\epsilon$ and $c$ in \eqref{eq:lower on eps} and \eqref{eq:def c}, respectively.
Here, we prove the difficult part of the theorem and show that for large enough values of $mn$, for any estimator there is a probability distribution under which with probability at least $1/3$ we have
\begin{equation}\label{eq:lower the 1/d bound}
\|\hat\theta - \theta^*\| \ge \epsilon/({2}c\sqrt{n}) = \frac1{640\times50^{1/d}\,d^{2.5}\,\log^{2+2/d}(mn)} \times\frac{1}{\sqrt{n}\,(mB)^{1/d}}.
\end{equation}
In order to draw a contradiction, suppose that there exists an estimator $\hat{E}_1$ such that in a system of $m$ machines and $n$ samples per machine, $\hat{E}_1$ has error less than $\epsilon/({2}c\sqrt{n})$ with probability at least $2/3$, for all distributions $P$ satisfying Assumptions~\ref{ass:1} and~\ref{ass:2}. 
Note that since $\hat{E}_1$ cannot beat the estimation error $1/\sqrt{mn}$ of the centralized solution, it follows that
\begin{equation} \label{eq:eps>sqrt m}
\epsilon \,\geq\, \frac{2c}{\sqrt{m}}\,\ge\, m^{-1/2}.
\end{equation}
We will show that $\epsilon=\tilde{\Omega}(m^{-1/d})$. 

We first improve the confidence of $\hat{E}_1$ via repetitions to obtain an estimator $\hat{E}_2$, as in the following lemma.
\begin{lemma} \label{lem:E2}
	There exists an estimator $\hat{E}_2$ such that in a system of $m\log^2(mn)$ machines and $n$ samples per machine, $\hat{E}_2$ has estimation error less than $\epsilon/(c\sqrt{n})$ with probability at least $1-\exp\big({-\Omega(\log^2 mn)\big)}$, for all distributions $P$ satisfying Assumptions~\ref{ass:1} and~\ref{ass:2}.
\end{lemma}
The proof is fairly standard, and is given in Appendix~\ref{app:E2}.

For any $f\in S^*$, consider a probability distribution $P\in \mathcal{C}$ such that: $P(f)=1/2$, $P_{i^+}=P_{i^-}=1/(4d)$, $i=1,\ldots,d$. 
Suppose that each machine observes $n$ samples from this distribution.
Let $n_0,n_{1^+},\ldots,n_{d^+},n_{1^-},\ldots, n_{d^-}$ be the number of that equal from  $f,g^+_1,\ldots, g^+_d,g^-_1,\ldots,g^-_d$, respectively.
We refer to $\underline{n}=[n_0,n_{1^+},n_{1^-},\ldots,n_{d^+},n_{d^-}]$ as the observed frequency vector of this particular machine.
We denote by $Y(f,\underline{n}^j)$ the signal generated by estimator $\hat{E}_2$ (equivalently by $\hat{E}_1$) at machine $j$, corresponding to the distribution $P$ and the observed frequency vector $\underline{n}^j$.


\begin{definition}
	Consider a system of $8m\log^2(mn)$ machines. 
	For any $f\in S^*$, we define $\mathcal{W}_f$ as the collection of pairs  $\big(\underline{n}^1,Y(f,\underline{n}^1)\big),\ldots,\big(\underline{n}^{8m\log^2(mn)},Y(f,\underline{n}^{8m\log^2(mn)})\big)$ that are generated via the above procedure.
\end{definition}

We now present the main technical lemma of this proof. It shows that employing $\hat{E}_2$, given $\mathcal{W}_f$, we can uniquely recover $f$ out of all functions in $S^*$, with high probability.
\begin{lemma}
	There exists an algorithm $\hat{E}_3$  that for any $f\in S^*$, given $\mathcal{W}_f$, it outputs an $h\in S^*$ such that $h=f$ with probability $1-\exp(-\Omega(\log^2(mn)))$. 
	\label{lemma:7}
\end{lemma}
\begin{proof}
	Consider the collection $\C$ of probability distributions defined in Definition~\ref{def:class C}.
	The high level idea is as follows. 
	We first show that for any distribution $P'\in \C$ corresponding to $f$, there is a sub-sampling of $\mathcal{W}_f$ such that the sub-sampled pairs are i.i.d.  and have distribution $P'$.
	As a result, employing estimator $\hat{E}_2$, we can find the minimizer of $\E_{g\sim P'}\big[g(\cdot)\big]$. 
	On the other hand, for any $x\in [-1/(c\sqrt{n}),1/(c\sqrt{n})]^d$, we can choose a $P'\in \C$  such that $x$ be the minimizer of $\E_{g\sim P'}\big[g(\cdot)\big]$, or equivalently $\nabla \E_{g\sim P'}\big[g(x)\big] = \nabla f(x)/2 + \nabla h(x) \simeq 0$, for some known linear function $h$ (see Definition~\ref{def:class C}).
	Therefore, the fact that $\hat{E}_2$ we can find the minimizer of $\E_{g\sim P'}\big[g(\cdot)\big]$ for all $P'\in \C$  implies that $\hat{E}_2$ has sufficient information to  obtain a good approximation of $\nabla f(x)$, for all  $x\in [-1/(c\sqrt{n}),1/(c\sqrt{n})]^d$, with high probability.
	This enables us to recover $\nabla f$ from $\mathcal{W}_f$ with high probability via some algorithm that we then call  $\hat{E}_3$. We now go over the details of the proof.
	
	Consider a cube $A=\big[-d/(c\sqrt{n}),d/(c\sqrt{n})\big]^d$ and 
	suppose that $A^*$ is a minimum $\epsilon\lambda/(4\sqrt{n})$-covering\footnote{By a covering, we mean a set $A^*$ such that for any $x\in A$, there is a point $p\in A^*$ such that $\|x-p\|\le \epsilon\lambda/(4\sqrt{n})$.} of $A$. 
	A regular grid yields a simple bound on the size of $A^*$:
	\begin{equation} \label{eq:size A*}
	|A^*|\leq (8\sqrt{d}/(c\lambda\epsilon))^d.
	\end{equation}
	Let $P\in\C$ be the probability distribution with $P(f)=1/2$ and $P_{i^+}=P_{i^-}=1/(4d)$, for $i=1,\ldots,d$. 
	Moreover, for any $v\in A^*$ consider the probability distributions $P^v\in \mathcal{C}$: $P^v(f)=1/2$, $P^v_{i^+}=1/(4d)+v_i/4$, and $P^v_{i^-}=1/(4d)-v_i/4$, for $i=1,\ldots,d$ (note that $P^v_{i^+},P^v_{i^+}\ge 0$, because $v\in A$ and $c>d$). 
	
	It follows from Lemma~\ref{Lemma:ratio} that for any observed frequency vector $\underline{n}$ in $\mathcal{W}_f$, we have with probability at least $1-\exp(-\Omega(\log^2(mn)))$,
	\begin{equation}
	\frac{1}{4}\leq \frac{P^v(\underline{n})}{2P(\underline{n})}\leq 1.
	\label{eq:cond}
	\end{equation}
	We sub-sample $\mathcal{W}_f$, and discard from $\mathcal{W}_f$ any pair  $\big(\underline{n},Y(f,\underline{n})\big)$ whose $\underline{n}$ does not satisfy \eqref{eq:cond}.  
	Otherwise, if $\underline{n}$ satisfies \eqref{eq:cond}, we then keep the pair $\big(\underline{n},Y(f,\underline{n})\big)$ with probability $P^v(\underline{n})/(2P(\underline{n}))$. 
	We denote the set of surviving samples by $\mathcal{W}_f^v$. 
	\begin{claim} \label{claim:subsample}
		With probability  $1-\exp\big(\Omega(-\log^2(mn))\big)$,  at least $m\log^2(mn)$ pairs $\big(\underline{n},Y(f,\underline{n})\big)$ survive the above sub-sampling procedure; these pairs are i.i.d and the corresponding $\underline{n}$'s have distribution $P^v$.
	\end{claim}
	The proof is given in Appendix~\ref{app:proof claim subsample}.

	Let $\hat{x}_v$ be the output of the server of estimator $\hat{E}_2$  to the input $\mathcal{W}_f^v$.
	It follows from  Lemma~\ref{lem:E2} and Claim~\ref{claim:subsample} that with probability at least $1-\exp\big(\Omega(-\log^2(mn))\big)$, 
	\begin{equation}
	\|\hat{x}_v-x^*_{v,f}\|\leq \frac{\epsilon}{c\sqrt{n}},
	\label{eq:xv}
	\end{equation}
	where $x^*_{v,f}$ is the minimizer of function $\mathbb{E}_{g\sim P^v}[g(x)]=\frac{1}{2}(f(x)+v^Tx)$. By repeating this process for different $v$'s, we compute $\hat{x}_v$ for all $v$ in $A^*$. We define event $\mathcal{E}$ as follows:
	\begin{equation*}
	\mathcal{E}: \,\, \|\hat{x}_v-x^*_{v,f}\|\leq \frac{\epsilon}{c\sqrt{n}},\,\, \forall v\in A^*. 
	\end{equation*}
	Then,
	\begin{align*}
	\Pr\big(\mathcal{E}\big)\,&\geq\, 1-|A^*|\exp(-\Omega(\log^2(mn)))\\
	&\geq\, 1-\left(\frac{8\sqrt{d}}{\lambda c\epsilon}\right)^d\,\exp\big(-\Omega(\log^2(mn))\big)\\
	&=\, 1-\Big(\frac{8\sqrt{d}}{4\lambda d}\Big)^d\Big(\frac{1}{\epsilon \log(mn)}\Big)^d \exp\big(-\Omega(\log^2(mn))\big)\\
	&\geq\, 1-\Big(\frac{2}{\lambda \sqrt{d}}\Big)^d \Big(\frac{\sqrt{m}}{ \log(mn)}\Big)^d\exp\big(-\Omega(\log^2(mn))\big)\\ 
	&=\,1-\exp\big(-\Omega(\log^2(mn))\big),
	\end{align*}
	where the first four relations follow from the union bound, \eqref{eq:size A*}, the definition of $c$ in \eqref{eq:def c}, and \eqref{eq:eps>sqrt m}, respectively.

	The algorithm $\hat{E}_3$ then returns, as its final estimation of $f$, an $\hat{f}\in S^*$ of the form
	\begin{equation}
	\hat{f}\in \argmin{g\in S^*}\, \max_{v\in A^*} \|\hat{x}_v- x^*_{v,g}\|
	\label{eq:fhat}
	\end{equation}

	
	We now bound the error probability of $\hat{E}_3$, and show that $ \hat{f} = f$ with probability at least $1-\exp(-\Omega(\log^2(mn)))$. 
	\begin{claim}
		For any $g\in S^*$ with $g\neq f$, there is a $v\in A^*$ such that  $\|x^*_{v,g}-x^*_{v,f}\|\geq \epsilon/(2\sqrt{n})$.
		\label{lemma:8}
	\end{claim}
	The proof is given in Appendix~\ref{app:claim A*}.

	Suppose that event $\mathcal{E}$ has occurred and consider a $g\in S^*$ with $g\neq f$. Then, it follows from  Claim~\ref{lemma:8} that there is a $v\in A^*$ such that 
	\begin{align*}
	\|\hat{x}_v-x^*_{v,g}\|\,&\geq\, \|x^*_{v,g}-x^*_{v,f}\|-\|x^*_{v,f}-\hat{x}_v\|\\
	&\geq^a\, \frac{\epsilon}{2\sqrt{n}}-\|x^*_{v,f}-\hat{x}_v\|\\
	&\geq^b\, \frac{\epsilon}{2\sqrt{n}}-\frac{\epsilon}{c\sqrt{n}}\\
	&>^c\,\frac{\epsilon}{c\sqrt{n}}.
	\end{align*}
	\\ ($a$) Due to Claim~\ref{lemma:8}.
	\\ ($b$) According to the definition of event $\mathcal{E}$.
	\\ ($c$) Based on the definition of $c=4d\log(mn)$ in \eqref{eq:def c}, we have $c>4$.\\
	Therefore, with probability at least $\Pr(\mathcal{E})=1-\exp\big(-\Omega(\log^2(mn))\big)$, for any $g\in S^*$ with $g\neq f$,
	\begin{equation*}
	\max_{v\in A^*} \|\hat{x}_v-x^*_{v,g}\| \,>\, \frac{\epsilon}{c\sqrt{n}} \,\geq\, \max_{v\in A^*} \|\hat{x}_v-x^*_{v,f}\|.
	\end{equation*}
	It then follows from \eqref{eq:fhat} that $\hat{f}=f$ with probability $1-\exp\big(-\Omega(\log^2(mn))\big)$.
	This shows that the error probability of $\hat{E}_3$ is $\exp\big(-\Omega(\log^2(mn))\big)$, and completes the proof of Lemma~\ref{lemma:7}.
\end{proof}

Back to the proof of Theorem \ref{th:lower c},  consider a random variable $X$ that has uniform distribution over the set of functions $S^*$ and 
a random variable $W$ over the domain $\{\mathcal{W}_f\}_{f\in S^*}$ with the following distribution:
\begin{equation*}
\Pr(W|X)=\Pr\big(\mathcal{W}_f=W \mid f=X\big).
\end{equation*}
Note that each $\mathcal{W}_f$  consists of at most $8m\log^2(mn)$ pairs, each containing a signal $Y$ of length $B$ bits and a vector $\underline{n}$ of $2d+1$ entries ranging over $[0, n]$. Therefore, each $\mathcal{W}_f$ can be expressed by a string of length at most $8m\log^2(mn)\big(B+ (2d+1)\log(n+1)\big)$ bits. As a result, we have the following upper bound on the size $|W|$ of the state space of random variable $X$:
\begin{equation} \label{eq:bound on |W|}
\log(|W|) \,\le\, 8m\log^2(mn)\Big(B+ (2d+1)\log(n+1)\Big).
\end{equation}
Based on Lemma \ref{lemma:7}, there exists an estimator $\hat{E}_3$ which observes $W$ and returns the correct $X$ with probability at least $1-\exp\big(-\Omega(\log^2(mn))\big)$. 
Let $\Pr(e)$ be the probability of error of this estimator. Then,
\begin{equation} \label{eq:Pe of E3}
\Pr(e) \,\le \, \exp\big(-\Omega(\log^2(mn))\big) \,<\, \frac{1}{2},
\end{equation}
for large enough $mn$.
On the other hand, it follows from the Fano's inequality in \eqref{eq:Fano unif} that
\begin{equation}\label{eq:p e lower 1/2}
\begin{split}
\Pr(e) \,&\geq \, 1-\frac{\log(|W|)+1}{\log(|X|)}\\
&\ge  \, 1-\frac{8m\log^2(mn)\Big(B+ (2d+1)\log(n+1)\Big)}{\log(|X|)}\\
&\ge  \, 1-\frac{25 mB\log^3(mn)}{\log(|X|)}\\
&=\, 1-\frac{25 mB\log^3(mn)}{K_{\epsilon/\sqrt{n},1/(c\sqrt{n})}}\\
&> \, 1- \frac{25 mB\log^3(mn)}{50 mB\log^3(mn)}\\
&=\, \frac12,
\end{split}
\end{equation}      
where the first equality is  by Definition~\ref{def:metric ent}, the second inequality follows from \eqref{eq:bound on |W|}, and the last inequality is due to \eqref{eq:lower k eps n}.
Eq. \eqref{eq:p e lower 1/2} contradicts \eqref{eq:Pe of E3}. 
Hence, our initial assumption that there is an estimator $\hat{E}_1$ with accuracy $\epsilon/({2}c\sqrt{n})$ and confidence $2/3$ is incorrect. 
This implies \eqref{eq:lower the 1/d bound} and completes the proof of Theorem~\ref{th:lower c}. 


\subsection{Proof of $\Omega\big(1/\sqrt{mn}\big)$ bound}\label{app:proof lem Eower c folklore}
We now go over the easier part of the theorem and show the $1/\sqrt{mn}$ lower bound on the estimation error. 
The $\Omega\left(1/\sqrt{mn}\right)$ barrier is actually well-known to hold in several centralized scenarios. 
Here we adopt the bound for our setting. 
Without loss of generality, suppose that the communication budget is infinite and the system is essentially centralized.
Consider functions $f_1,f_2\in\F$, such that $f_1(\theta)=\|\theta-\one\|^2/4\sqrt{d}$ and $f_1(\theta)=\|\theta+\one\|^2/4\sqrt{d}$, for  all $\theta\in[-1,1]^d$.
We define two probability distributions $P_1$ and $P_2$ as follows
\begin{equation*}
P_1:\,\begin{cases}
\Pr(f_1)=1/2-1/5\sqrt{mn},&\\
\Pr(f_2)=1/2+1/5\sqrt{mn},&
\end{cases}
\end{equation*}
\begin{equation*}
P_2:\,\begin{cases}
\Pr(f_1)=1/2+1/5\sqrt{mn},&\\
\Pr(f_2)=1/2-1/5\sqrt{mn}.&
\end{cases}
\end{equation*}
Then, the minimizers of $\Exp_{f\sim P_1}\big(f(\cdot)\big)$ and $\Exp_{f\sim P_2}\big(f(\cdot)\big)$ are $\theta_1\triangleq -\one/5\sqrt{mn}$ and $\theta_2\triangleq \one/5\sqrt{mn}$, respectively.
Therefore, $\|\theta_1-\theta_2\| = 2\sqrt{d}/5\sqrt{mn}$.

We now show that no estimator has estimation error less than $\sqrt{d}/5\sqrt{mn}$ with probability at least $2/3$.
In order to draw a contradiction, suppose that there is an estimator $\hat{E}$ for which $\|\hat\theta-\theta^*\|<\sqrt{d}/5\sqrt{mn}$ with probability at least $2/3$. 
Based on $\hat{E}$ we devise a binary hypothesis test  $\cal{T}$ as follows. This $\cal{T}$ tests between hypothesis $\mathcal{H}_1$ that sample are drown from distribution $P_1$ and hypothesis $\mathcal{H}_2$ that sample are drown from distribution $P_2$.
It works as follows: for the output $\hat\theta$ of $\hat{E}$, $\cal{T}$ chooses $\mathcal{H}_1$ if $\|\hat\theta-\theta^1\|<\|\hat\theta-\theta^2\|$, and  $\mathcal{H}_2$ otherwise.
Since $\|\hat\theta-\theta^*\|<\sqrt{d}/5\sqrt{mn}$ with probability at least $2/3$, it follows that the error success of $\cal{T}$ is at least $2/3$. 
This contradicts Lemma~\ref{lem:hypo test}, according to which no binary hypothesis test can distinguish between $P_1$ and $P_2$ with probability at least $2/3$.
Therefor, there is no estimator for which  $\|\hat\theta-\theta^*\|< \sqrt{d}/5\sqrt{mn}$ with probability at least $2/3$.
Combined with \eqref{eq:lower the 1/d bound}, this completes the proof of Theorem~\ref{th:lower c}.


\medskip
\section{Proof of Lemma \ref{lemma:metric} } \label{app:proof metric ent}
We assume
\begin{equation}\label{eq:range of delta}
10\sqrt{d} \epsilon \leq \delta\leq 1,
\end{equation}
and show that 
\begin{equation*}
K_{\epsilon,\delta}\geq \Big(\frac{1}{20\sqrt{d}}\Big)^d \Big(\frac{\delta}{\epsilon}\Big)^d.
\end{equation*}
We begin by a claim on existence of a kernel function with certain properties.
\begin{claim} \label{claim:kernel}
	There exists a continuously twice differentiable function $h:\mathbb{R}^d\rightarrow \mathbb{R}$ with the following properties:
	\begin{align}
	h(x) &= 0,&\textrm{for}\quad &x\not\in (-1,1)^d, \label{eq:kernel 1}\\
	|h(x)| &\le 1,&\textrm{for}\quad &x\in \mathbb{R}^d \label{eq:kernel 1.5}\\
	\|\nabla h(0)\|&>\frac12, & &  \label{eq:kernel 2}\\
	\|\nabla h(x)\| & \leq 3,&\textrm{for}\quad  &x\in \mathbb{R}^d, \label{eq:kernel 3}\\
	-4I_{d\times d} \preceq \nabla^2 h(x) &\preceq 4I_{d \times d},\hspace{-2cm}&\textrm{for}\quad & x\in \mathbb{R}^d. \label{eq:kernel 4}
	\end{align}
\end{claim}

\begin{proof}
	We show that the following function satisfies \eqref{eq:kernel 1}--\eqref{eq:kernel 4}:
	\begin{equation}
	h(x)=
	\begin{cases}
	\frac{8}{27}\Big(1-\frac{9}{4}\, \|x+\frac{1}{3}e_1\|^2\Big)^3, &\mbox{if }\,\|x+\frac{1}{3}e_1\|_2\leq \frac{2}{3},\\
	0, &\mbox{ otherwise,}
	\end{cases}
	\label{eq:tri}
	\end{equation}
	where $e_1$ is a vector whose first entry equals one and all other entries equal zero. 
	Note that if $\|x+e_1/3\|\ge 2/3$, then we have $h(x)=0$, $\nabla h(x)=\mathbf{0}_{n\times 1}$, and $\nabla^2 h(x)=\mathbf{0}_{n\times n}$. Therefore, the function value and its the first and second derivatives are continuous. Hence, $h$ is continuously twice differentiable.
	The gradient and Hessian of function $h$ are as follows:
	\begin{align}
	\nabla h(x)&= -4\Big(1-\frac{9}{4}\|x+\frac{1}{3}e_1\|^2\Big)^2\Big(x+\frac{1}{3}e_1\Big) \label{eq:h grad}\\
	\nabla^2 h(x)&=36 \Big(1-\frac{9}{4} \|x+\frac{1}{3}e_1\|^2\Big)\big(x+e_1/3\big)\big(x+e_1/3\big)^T-4\Big(1-\frac{9}{4}\|x+\frac{1}{3}e_1\|^2\Big)^2 I. \label{eq:h hess}
	\end{align}
	
	We now examine  properties \eqref{eq:kernel 1}--\eqref{eq:kernel 4}. 
	For \eqref{eq:kernel 1}, note that if $x\not\in (-1,1)^d$, then $\|x+e_1/3\|\ge 2/3$, and as a result,  $h(x)=0$.
	Eq. \eqref{eq:kernel 1.5} is immediate from the definition of $h$ in \eqref{eq:tri}.
	Property \eqref{eq:kernel 2} follows because $\|\nabla h(0)\| \,=\, 3/4\,>\, 1/2$.
	For \eqref{eq:kernel 3}, consider any $x$ such that $\nabla h(x)\ne 0$. Then, $\|x+e_1/3\|\leq 2/3$, and \eqref{eq:h grad} implies that
	\begin{equation*}
	\|\nabla h(x)\|\,=\, 4\left(1-\frac{9}{4}\|x+\frac{1}{3}e_1\|^2\right)^2\|x+e_1/3\|
	\,\leq\, 4\times \frac{2}{3}
	\,<\, 3.
	\end{equation*}
	Based on \eqref{eq:h hess}, at any point $x$ where $\|x+e_1/3\|<2/3$, the largest and smallest eigenvalues of the Hessian matrix are
	\begin{equation}
	\begin{split}
	\lambda_{min} & =-4\Big(1-\frac{9}{4}\|x\|^2\Big)^2\geq -4,\\
	\lambda_{max} & =36\Big(1-\frac{9}{4}\|x+e_1/3\|^2\Big)\|x+e_1/3\|^2-4\Big(1-\frac{9}{4}\|x+e_1/3\|^2\Big)^2.\\
	\end{split}
	\label{eq:a41}
	\end{equation}
	Letting $\alpha=\big(3\|x+e_1/3\|/2\big)^2$, we have
	\begin{equation}
	\lambda_{max}\,\leq\, \sup_{\alpha\in [0,1]} 16(1-\alpha)\alpha -4(1-\alpha)^2
	\,\leq\, \sup_{\alpha\in [0,1]} 4-4(1-\alpha)^2
	\,\leq\, 4.
	\label{eq:a42}
	\end{equation}
	Property \eqref{eq:kernel 4}  follows from \eqref{eq:a41} and \eqref{eq:a42}. This completes the proof of the claim.
\end{proof}
\medskip

Consider a function  $h$ as in Claim~\eqref{claim:kernel}, and let
$$k(x) = 10\sqrt{d} \epsilon^2 \, h\left(\frac{x}{10\sqrt{d}\epsilon}\right),$$
for all $x\in\R^d$.
Also let $\epsilon' = 10\sqrt{d}\epsilon$. Then, \eqref{eq:range of delta} implies that $\delta\ge\epsilon'$.
It follows from Claim~\ref{claim:kernel} that $k(\cdot)$ is continuously twice differentiable, and
\begin{align}
k(x) &= 0,&\textrm{for}\quad &x\not\in (-\epsilon',\epsilon')^d, \label{eq:k kernel 1}\\
|k(x)| &\le 10\sqrt{d} \epsilon^2,&\textrm{for}\quad &x\in \mathbb{R}^d \label{eq:k kernel 1.5}\\
\|\nabla k(0)\|&>\frac{\epsilon}2, & &  \label{eq:k kernel 2}\\
\|\nabla k(x)\| & \leq 3\epsilon,&\textrm{for}\quad  &x\in \mathbb{R}^d, \label{eq:k kernel 3}\\
-\frac{4}{10\sqrt{d}}\,I_{d\times d} \preceq \nabla^2 k(x) &\preceq \frac{4}{10\sqrt{d}}\,I_{d \times d},\hspace{-2cm}&\textrm{for}\quad & x\in \mathbb{R}^d. \label{eq:k kernel 4}
\end{align}

Consider a regular $2\epsilon'$-grid $G$ inside $[-\delta,\delta]^d$, such that the coordinates of points in $G$ are odd multiples of $\epsilon'$, e.g., $[\epsilon',\epsilon',\ldots,\epsilon']\in G$. Let $M$ be the collection of all functions  $s:G\rightarrow \{-1,+1\}$ that assign $\pm 1$ values to each grid point in $G$. 
Therefore , $\M$ has size
\begin{equation}
|\M| \,=\, 2^{|G|} \,=\, 2^{\left(2\left\lfloor \frac{\delta+\epsilon'}{2\epsilon'}\right\rfloor\right)^d}
\,\geq\, 2^{\left(\frac{\delta}{2\epsilon'}\right)^d},
\label{eq:M}
\end{equation}
where in the last inequality is because $\delta\ge \epsilon'$.

For any function $s\in \M$, we define a function $f_s:\R^d\to\R$ of the following form
\begin{equation} \label{def:function fs}
f_s(x)\,=\,\left(\sum_{p\in G} s(p)k(x-p)\right)+\frac{1}{4\sqrt{d}}\|x\|^2.
\end{equation}
There is an  equivalent representation for $f_s$ as follows.  
For any $x\in \mathbb{R}^d$, let $\pi(x)$ be the closest point to $x$ in $G$.
Then, it follows from \eqref{eq:k kernel 1} that for any $x\in \mathbb{R}^d$,
\begin{equation}
f_s(x) \,=\, s\big(\pi(x)\big) \, k\big(x-\pi(x)\big)\,+\,\frac{1}{4\sqrt{d}}\|x\|^2.
\label{eq:anotherform}
\end{equation}

\begin{claim}
	For any $s\in \M$, we have $f_s\in \Ft$, $f_s(\mathbf{0})=0$, and $\nabla f_s(\mathbf{0})=\mathbf{0}$.
\end{claim}
\begin{proof}
	First note that $\pi(0)=[\epsilon',\ldots,\epsilon']=\epsilon'\mathbf{1}$, and it follows from \eqref{eq:anotherform} that 
	\begin{equation*}
	\begin{split}
	f_s(\mathbf{0})\,&=\,k(\epsilon'\mathbf{1})=0,\\ 
	\nabla f(\mathbf{0}) \,&=\, \nabla k(\epsilon'\mathbf{1}) \,=\, \mathbf{0},
	\end{split}
	\end{equation*} 
	where the second and last equalities are due to \eqref{eq:k kernel 1}. 
	Moreover, since $k(\cdot)$ is continuously twice differentiable, so is $f_s$.
	We now show that $f_s\in \Ft$, i.e.,
	\begin{align}
	|f_s(x)| \,&\leq\, \sqrt{d},\label{eq:prop fs 1}\\
	\|\nabla f_s(x)\| \, &\leq\, 1,\label{eq:prop fs 2}\\
	\lambda I_{d\times d} \,\preceq\, \nabla^2 f_s(x)  \,&\preceq\, \frac{1}{\sqrt{d}} I_{d\times d}, \label{eq:prop fs 3}
	\end{align}
	for all $x\in[-1,1]^d$

	From \eqref{eq:anotherform} and \eqref{eq:k kernel 1.5} , we have for $x\in [-1,1]^d$,
	\begin{equation*}
	|f_s(x)| \, \le\, |k(x)| + \frac{1}{4\sqrt{d}}\|x\|^2 \, \le\, 10\sqrt{d} \epsilon^2 + \frac{\sqrt{d}}{4} \,\le^{a}\, \epsilon + \frac{\sqrt{d}}{4} \, \le^{b}\, \sqrt{d},
	\end{equation*}
	and
	\begin{equation*}
	\begin{split}
	\|\nabla f_s(x)\|\,&=\, \Big\|\nabla k(x-\pi(x)) +\frac{x}{2\sqrt{d}} \Big\|\\
	&\leq\, \|\nabla k(x-\pi(x))\|+\frac{\|x\|}{2\sqrt{d}}\\
	&\leq^c 3\epsilon+\frac{\sqrt{d}}{2\sqrt{d}}\\
	&\leq^d 1,
	\end{split}
	\end{equation*}
	where $(a)$, $(b)$, and $(d)$  are due to the assumption $10\sqrt{d}\epsilon\leq 1$ in \eqref{eq:range of delta}, and   $(c)$ follows from \eqref{eq:k kernel 3}.
	For \eqref{eq:prop fs 3}, we have from \eqref{eq:anotherform},
	\begin{equation*}
	\nabla^2 f(x) \,=\, s\big(\pi(x)\big)\nabla^2 k\big(x-\pi(x)\big)+\frac{1}{2\sqrt{d}}I.
	\end{equation*}
	It then follows from \eqref{eq:k kernel 4}  that for any $x\in [-1,1]^d$, 
	\begin{equation*}
	\lambda I \,\preceq\, \frac{1}{10\sqrt{d}}I \,=\, \frac{-4}{10\sqrt{d}}I+\frac{1}{2\sqrt{d}} 
	\,\preceq\, \nabla^2 f(x) 
	\,\preceq\, \frac{4}{10\sqrt{d}}I + \frac{1}{2\sqrt{d}}I
	\,\prec\, \frac{1}{\sqrt{d}}I.
	\end{equation*}
	where the first inequality is from the assumption $\lambda \le {1}/{\big(10\sqrt{d}\big)}$ in the statement of Theorem~\ref{th:lower c}.
	Therefore $f_s\in \Ft$ and the claim follows.
\end{proof}

Let $\S_{\epsilon,\delta}$ be the collection of functions $f_s$ defined in \eqref{def:function fs}, for all $s\in \M$; i.e.,
\begin{equation}
\S_{\epsilon,\delta} \,=\,  \big\{ f_s \, \mid\, s\in \M \big\}.
\label{eq:defS}
\end{equation}
We  show that $\S_{\epsilon,\delta}$ is an $(\epsilon,\delta)$-packing. 
Consider a pair of distinct functions $s_1,s_2\in M$.  
Then, there exists a point $p\in G$ such that $s_1(p)\neq s_2(p)$; equivalently, $|s_1(p)-s_2(p)|=2$. 
Therefore,
\begin{equation*}
\begin{split}
\|\nabla f_{s_1}(p)-\nabla f_{s_2}(p)\|\,&=\, \|s_1(p)\nabla k(0)-s_2(p)\nabla k(0)\|\\
&=\,|s_1(p)-s_2(p)|\times \|\nabla k(0)\|\\
&>\, 2\times \frac{\epsilon}{2}\\
&=\,\epsilon,
\end{split}
\end{equation*}
where the first equality is due to \eqref{eq:anotherform}, and the inequality follows from \eqref{eq:k kernel 3}.
This shows that 
for any pair of distinct $f,g\in S_{\epsilon,\delta}$ functions, $\|\nabla f(p)-\nabla g(p)\|\geq \epsilon$, for some $p\in [-\delta,\delta]^d$. 
Therefore, $S_{\epsilon,\delta}$ is an $(\epsilon,\delta)$-packing. 

Finally, it follows from \eqref{eq:M} that
\begin{equation*}
K_{\epsilon,\delta} \,\geq\, \log\big(|S_{\epsilon,\delta}|\big)
\,=\,\log\big(|M|\big)
\,\geq\, \left(\frac{\delta}{2\epsilon'}\right)^d
\,=\, \left(\frac{1}{20\sqrt{d}}\right)^d \, \left(\frac{\delta}{\epsilon}\right)^d,
\end{equation*}
and Lemma~\ref{lemma:metric} follows.


\medskip
\section{Proof of Lemma~\ref{Lemma:ratio}} \label{app:proof flatness}
We define an event $\mathcal{E}_0$ as follows:
\begin{equation*}
\Big|n_{i^+}-\frac{n}{4d}\Big|\leq \frac{\sqrt{n}\log(mn)}{16d},\quad \textrm{and}\quad \Big|n_{i^-}-\frac{n}{4d}\Big|\leq \frac{\sqrt{n}\log(mn)}{16d},\quad i=1,\cdots,d. 
\end{equation*}
We first show that $\mathcal{E}_0$ occurs with high probability.
Let $n_i$ be the $i$-th entry of $\underline{n}$ for $i\geq 1$.  Then, 
\begin{align*}
\Pr_{\underline{n}\sim P} \Bigg( \Big|n_i-\frac{n}{4d}\Big|\geq \frac{\sqrt{n}\log(mn)}{16d}\Bigg)\,
&\leq\, \Pr_{\underline{n}\sim P} \Bigg(\Big|n_i-\mathbb{E}_P[n_i]\Big|+\Big|\mathbb{E}_P[n_i]-\frac{n}{4d}\Big|\geq \frac{\sqrt{n}\log(mn)}{16d}\Bigg)\\
&\leq^a\,\Pr_{\underline{n}\sim P}\Bigg(\Big|n_i-\mathbb{E}_P[n_i]\Big|\geq \frac{\sqrt{n}}{16d}\Big(\log(mn)-\frac{2}{\log(mn)}\Big)\Bigg)\\
&\leq^b\, \Pr_{\underline{n}\sim P} \Bigg( \big|n_i -\mathbb{E}_P[n_i]\Big|\geq \frac{\sqrt{n}\log(mn)}{32d}\Bigg)\\
&=\, \Pr_{\underline{n}\sim P} \Bigg( \frac{1}{n}\Big|n_i -\mathbb{E}_P[n_i]\Big|\geq \frac{\log(mn)}{32d\sqrt{n}}\Bigg)\\
&\leq^c\, \exp\left(-2n\left(\frac{\log(mn)}{32\sqrt{n}d}\right)^2\right)\\
&=\,\exp\Big(-\Omega\big(\log^2(mn)\big)\Big),
\end{align*}
where ($a$) is due to that fact that $\big|\mathbb{E}[n_i]-n/(2d)\big|=\big|P_in-{n}/({4d})\big|\leq 
n/(2c\sqrt{n})=\sqrt{n}/\big(8d\log(mn)\big)$, ($b$) is valid for $mn\ge4$, and ($c$) follows from the Hoeffding's inequality. 
It then follows from the union bound that
\begin{equation} \label{eq:pr E1 in PP'}
\Pr(\mathcal{E}_0)\,\geq\, 1-d\exp\Big(-\Omega\big(\log^2(mn)\big)\Big) \,=\, 1-\exp\Big(-\Omega\big(\log^2(mn)\big)\Big).
\end{equation}

Consider a pair $\underline{n}$ and $\underline{n}'$ of vectors that satisfy $\mathcal{E}_0$.
Then,
\begin{equation*}
\begin{split}
\frac{P\big(\underline{n}\big)/P'\big(\underline{n}\big)}{P\big(\underline{n}'\big)/P'\big(\underline{n}'\big)}\,
&=\,\frac{\prod_{i=0}^{2d}\big(P_i/P'_i\big)^{n_i}}{\prod_{i=0}^{2d}\big(P_i/P'_i\big)^{n'_i}}\\
&=\, \prod_{i=0}^{2d} \left(\frac{P_i}{P'_i}\right)^{n_i-n'_i}\\
&=^a\, \prod_{i=1}^{2d} \left(\frac{P_i}{P'_i}\right)^{n_i-n'_i}\\
&\leq^b\, \prod_{i=1}^{2d} \left(\frac{1/(4d)+\big(2c\sqrt{n}\big)}{1/(4d) -1/\big(2c\sqrt{n}\big)}\right)^{2\sqrt{n}\log(mn)/(16d)}\\
&=\, \prod_{i=1}^{2d} \left(\frac{1+2d/\big(c\sqrt{n}\big)}{1 -2d/\big(c\sqrt{n}\big)}\right)^{\sqrt{n}\log(mn)/(8d)}\\
&\le^c\, \prod_{i=1}^{2d}\left(1+4\times\frac{2d}{c\sqrt{n}}\right)^{\sqrt{n}\log(mn)/(8d)}\\
&\leq^d\, \exp\left(2d\frac{8d}{c\sqrt{n}}\times \frac{\sqrt{n}\log(mn)}{8d}\right)\\
&=\,\exp\left(\frac{16d^2\log(mn)}{32d^2\log(mn)}\right)\\
&=\,\sqrt{e},
\end{split} 
\label{eq:ratio}
\end{equation*}
($a$) Due to the fact that: $P_0=P'_0=1/2$.\\
($b$) Since $\underline{n}$ and $\underline{n}'$ satisfy event $\mathcal{E}_0$ and $P,P'\in \mathcal{C}$.\\
($c$) Because $2d/c \le 1/2$.\\
($d$) Due to the fact that $1+x\leq \exp(x)$.\\
Therefore,
\begin{align*}
\sup_{\underline{n}\in \mathcal{E}_0} \frac{P\big(\underline{n}\big)}{P'\big(\underline{n}\big)} \,
&\leq\, \sqrt{e}\,\, \inf_{\underline{n}\in \mathcal{E}_0} \frac{P\big(\underline{n}\big)}{P'\big(\underline{n}\big)}\\
&\leq\, \sqrt{e}\,\,\frac{\sum_{\underline{n}\in \mathcal{E}_0}P\big(\underline{n}\big)}{\sum_{\underline{n}\in \mathcal{E}_0}P'\big(\underline{n}\big)}\\ 
&= \, \sqrt{e}\,\,\frac{P(\mathcal{E}_0)}{P'(\mathcal{E}_0)}\\
&\leq\, \sqrt{e} \Big(1-\exp\big(-\Omega(\log^2(mn))\big)\Big)\\
&\le\, 2, 	\end{align*}
where the inequality before the last one is due to \eqref{eq:pr E1 in PP'} and the last inquality is valid for sufficiently large values of $mn$. 
Interchanging the roles of $P$ and $P'$, it follows that
$\inf{\underline{n}\in \mathcal{E}_0} {P\big(\underline{n}\big)}/{P'\big(\underline{n}\big)}\ge 1/2$.
This completes the proof of Lemma~\ref{Lemma:ratio}.

\medskip

\section{Proof of Lemma~\ref{lem:E2}} \label{app:E2}
The server subdivides the $m\log^2(mn)$ machines into $\log^2(mn)$ groups of $m$ machines, and employs  $\hat{E}_1$ to obtain an estimate $\hat{\theta}_i$ for each group $i=1,\ldots,\log^2(mn)$.
Let $k= \log(mn)^2$ and without loss of generality suppose that $k$ is an even integer.
Consider a $d$-dimensional ball ${B}$ of smallest radius that encloses at least $k/2 +1$  points from $\hat{\theta}_1,\ldots, \hat{\theta}_k$.
Let $\hat{\theta}$ be the center of $B$.
The estimator $\hat{E}_2$ then outputs $\hat{\theta}$ as an estimation of $\theta^*$.

We now show that $\| \hat{\theta}-\theta^* \| \le \epsilon/(c\sqrt{n})$ with high probability.
Let $B'$ be the ball of radius $\epsilon/({2}c\sqrt{n})$ centered at $\theta^*$, and let $q$ be the number of point from $\hat{\theta}_1,\ldots, \hat{\theta}_k$ that lie in $B'$. 
Since the error probability of $\hat{E}_1$ is less than $1/3$, we have $\E[q]\ge 2k/3$.
Then, 
\begin{equation*}
\begin{split}
\Pr\big( q\le k/2  \big) &\,\le \, \Pr\big( q -\E[q] \le  -k/6  \big) \\
&\,\le\, \exp\left( \frac{-2 k}{36}  \right)\\
&\, = \, \exp\left( \frac{- \log^2(mn)}{18}  \right)\\
&\, = \, \exp\Big( -\Omega\big( \log^2(mn)\big)  \Big),
\end{split}
\end{equation*}
where the second inequality follows from the Hoeffding's inequality.

Therefore, with probability $1-\exp\big({-\Omega(\log^2 mn)\big)}$, $B'$ encloses at least $k/2+1$ points from $\hat{\theta}_1,\ldots, \hat{\theta}_k$. 
In this case, by definition, the radius $r$ of $B$ would be no larger that the radius fo $B'$, i.e., $r\le\epsilon/(2c\sqrt{n}) $.
Moreover, since $B$ and $B'$ each encapsulate at least $k/2+1$ points out of $k$ points, they intersect (say at point $p$). 
Then, with probability at least $1-\exp\big({-\Omega(\log^2 mn)\big)}$,
\begin{equation*}
\|\hat{\theta} - \theta^*  \| \, \le \, \|\hat{\theta} - p  \| + \|p - \theta^*  \| \, \le \, r + \frac{\epsilon}{2c\sqrt{n}}
\, \le\, \frac{\epsilon}{c\sqrt{n}},
\end{equation*}
and the lemma follows.

\medskip

\section{Missing Parts of the Proof of Lemma~\ref{lemma:7}}
\subsection{Proof of Claim~\ref{claim:subsample}} \label{app:proof claim subsample}
Since  
$\mathcal{W}_f$ consists of $8m\log^2(mn)$ samples, the union bound implies that all received $\underline{n}$'s satisfy  \eqref{eq:cond} with probability at least  $1-8m\log^2(mn)\exp\big(-\Omega(\log^2(mn))\big)=1-\exp\big(-\Omega(\log^2(mn))\big)$. 
Thus, with probability at least $1-\exp\big(-\Omega(\log^2(mn))\big)$, the survived sub-sampled pairs are i.i.d. with distribution $P^v$. 
Moreover, if we denote the number of sub-sampled pairs by $t$, then,
\begin{equation*}
\mathbb{E}[t]\,\geq\, \Bigg(\frac{1}{4}-\exp\Big(-\Omega(-\log^2(mn))\Big)\Bigg)\times 8m\log^2(mn)
\,\geq\, \frac{3}{2}m\log^2(mn),
\end{equation*}
for large enough values of $mn$. Then,
\begin{equation}
\begin{split}
\Pr\big(t<m\log^2(mn)\big)\,&\leq\, \Pr\left(t-\mathbb{E}[t]\leq -\frac{1}{2}m\log^2(mn)\right)\\
& \hspace{0cm} \leq\,  \Pr\left(\frac{t-\mathbb{E}[t]}{8m\log^2(mn)}\leq -\frac{1}{16}\right)\\
& \hspace{0cm} \leq\, \exp\left(-16m\log^2(mn)\left(\frac1{16}\right)^2\right)\\
&\hspace{0cm} = \, \exp\Big(-\Omega\big(\log^2(mn)\big)\Big),
\end{split}
\label{eq:boundt}
\end{equation}
where the last inequality is due to  Hoeffding's inequality. This completes the proof of the claim.

\medskip
\subsection{Proof of Claim~\ref{lemma:8}} \label{app:claim A*}
According to the definition of $S^*$ and the fact that $g,f\in S^*$ and $f\neq g$, there exists $x\in [-1/(c\sqrt{n}),1/(c\sqrt{n})]^d$ such that:
\begin{equation}
\|\nabla f(x)-\nabla g(x)\|\geq \epsilon/\sqrt{n}.
\label{eq:lemma8:1}
\end{equation}
It follows from the assumption $\nabla f(0)=0$ that
\begin{equation*}
\|\nabla f(x)\|\,=\,\|\nabla f(x)-\nabla f(0)\|\leq^a \|x\|\,\leq\, \frac{\sqrt{d}}{c\sqrt{n}},
\end{equation*}
where the first inequality is due to the Lipschitz continuity of derivative of $f$ (cf. Assumption~\ref{ass:1}).
Then, 
\begin{equation}  
-\nabla f(x)\in A.
\label{eq:grad in A}
\end{equation}

Let $v\in A^*$ be the closest point of $A^*$ to $-\nabla f(x)$.
Since $A^*$ is an $\big(\epsilon\lambda/(4\sqrt{n})\big)$-covering of $A$, it follows from \eqref{eq:grad in A} that
\begin{equation}
\|\nabla f(x)+v\| \,\leq\, \frac{\epsilon \lambda}{4\sqrt{n}}.
\label{eq:lemma8:2}
\end{equation}

According to Assumption~\ref{ass:1}, $f$ is $\lambda$-strongly convex. Then,
\begin{equation}
\|x-x^*_{v,f}\|\,\leq \, \frac{\|\nabla f(x) - \nabla f(x^*_{v,f})\|}{\lambda} \, =\, \frac{\|\nabla f(x)+v\|}{\lambda}
\,\leq \,\frac{\epsilon \lambda/(4\sqrt{n})}{\lambda}
\,=\, \frac{\epsilon}{4\sqrt{n}},
\label{eq:lemma8:3}
\end{equation}
where the first equality is because $x^*_{v,f}$ is a minimizer of $f(x) + v^Tx$.
On the other hand, it follows from \eqref{eq:lemma8:1} and \eqref{eq:lemma8:2} that
\begin{equation*}
\begin{split}
\|\nabla g(x)+v\|\,&\geq\, \|\nabla g(x)-\nabla f(x)\|-\|\nabla f (x)+v\|\\
&\geq\, \frac{\epsilon}{\sqrt{n}}-\frac{\epsilon \lambda}{4\sqrt{n}}\\
& \geq\, \frac{\epsilon}{\sqrt{n}}-\frac{\epsilon}{4\sqrt{n}}\\
& =\, \frac{3\epsilon}{4\sqrt{n}},\\
\end{split}
\end{equation*}
where the last inequality is because $\nabla f$ is Lipschtz continuous with constant $1$ and as a result, $\lambda\le 1$.
Then,
\begin{equation}
\|x^*_{v,g}-x\|\geq \|\nabla g(x)+v\|
\,\geq\,\frac{3\epsilon}{4\lambda\sqrt{n}}
\,\ge\,\frac{3\epsilon}{4\sqrt{n}}.
\label{eq:lemma8:4}
\end{equation}
Combining \eqref{eq:lemma8:3} and \eqref{eq:lemma8:4}, we obtain
\begin{equation*}
\|x^*_{v,g}-x^*_{v,f}\| \,\geq\, \|x^*_{v,g}-x\| - \|x^*_{v,f}-x\| \,\geq\,  \frac{3\epsilon}{4\sqrt{n}}-\frac{\epsilon}{4\sqrt{n}} \,=\, \frac{\epsilon}{2\sqrt{n}}.
\end{equation*}
This completes the proof of Claim~\ref{lemma:8}.


\medskip
\section{Proof of Theorem~\ref{th:main upper c}} \label{sec:proof main alg c}
We first show that $s^*$ is a closest grid point of $G$ to $\theta^*$ with high probability. 
We then show that for any $l\le t$ and any $p\in \tilde{G}_{s^*}^l$, the number of sub-signals corresponding to $p$ after redundancy elimination  is large enough so that the server obtains a good 
approximation of $\nabla F$ at $p$. 
Once we have a good approximation of $\nabla F$ at all points of $\tilde{G}_{s^*}^t$, a point with the minimum norm for this approximation lies close to the minimizer of $F$. 

Recall the definition of $s^*$ as the grid point of $G$ that appears for the most number of times in the $s$ component of the received signals. 
Let $m^*$ be the number of machines that select $s=s^*$. 
We let $\mathcal{E'}$  be the event that $m^*\geq m/2^d$ and $\theta^*$ lies in the $\big(2\log(mn)/\sqrt{n}\big)$-cube $C_{s^*}$ centered at $s^*$, i.e.,
\begin{equation} \label{eq:s* in D}
\|s^*-\theta^*\|_\infty \leq \frac{\log(mn)}{\sqrt{n}}.
\end{equation}
Then, 
\begin{lemma} \label{lem:pro E1 c}
	\begin{equation}\label{eq:prb E1 upper}
	\Pr\big(\mathcal{E'}\big)=1-\exp\big(-\Omega(\log^2(mn))\big).
	\end{equation}
\end{lemma}
The proof relies on concentration inequalities, and is given in Appendix~\ref{app:proof lem EM1}.

We now turn our focus to the inside of cube $C_{s^*}$.
Let 
\begin{equation} \label{eq:def eps upper}
\epsilon \,\triangleq\, \frac{2\sqrt{d}\,\log(mn)}{\sqrt{n}}\times \delta \,=\, 
{4d^{1.5} \log^{4}(mn)} \, \max \left(\frac1{\sqrt{n}\,(mB)^{1/{d}}}, \, \frac{2^{d/2}}{\sqrt{mn}} \right). 
\end{equation}
For any $p\in \bigcup_{l\le t} \tilde{G}_{s^*}^l$, let $N_p$ be the number of machines that select point $p$ in at least one of their sub-signals. Equivalently,  $N_p$ is the number of sub-signals after redundancy elimination that have point $p$ as their second argument.
Let $\mathcal{E''}$ be the event that for any $l\le t$ and any $p\in \tilde{G}_{s^*}^l$, we have 
\begin{equation} \label{eq:bound Np}
N_p \ge \frac{4d^22^{-2l}\log^6(mn)}{n\epsilon^2}.
\end{equation}
Then,
\begin{lemma} \label{lem:prob of E2 upper}
	$\Pr\big(\mathcal{E''} \big) = 1-\exp\big(-\Omega(\log^2(mn))\big)$.
\end{lemma}
The proof is based on the concentration inequality in Lemma~\ref{lemma:CI}~(b), and is given in Appendix~\ref{app:proof lem E2 upper}.

Capitalizing on Lemma~\ref{lem:prob of E2 upper}, we now obtain a bound on the estimation error of gradient of $F$ at the grid points in $\tilde{G}_{s^*}^l$. 
Let $\mathcal{E'''}$ be the event that for any $l\le t$ and any grid point $p\in \tilde{G}_{s^*}^l$, we have $$\big\|\hat{\nabla} F(p) -\nabla F(p)\big\|\,<\,\frac{\epsilon}{4}.$$

\begin{lemma} \label{lem:prob E3 upper}
	$\Pr\big(\mathcal{E'''}\big)\,=\,1-\exp\big(-\Omega(\log^2(mn))\big)$. 
\end{lemma}
The proof is given in Appendix~\ref{app:proof E3 upper} and relies on Hoeffding's inequality and the lower  bound on the number of received signals for each grid point, driven in Lemma~\ref{lem:prob of E2 upper}.

In the remainder of the proof, we assume that \eqref{eq:s* in D} and  $\mathcal{E'''}$ hold. 
Let $p^*$ be the  closest grid point in $\tilde{G}_{s^*}^t$ to $\theta^*$.
Therefore, 
\begin{equation}\label{eq:p theta eps2}
\|p^*-\theta^*\|\,\leq\, \sqrt{d}\,2^{-t}\frac{\log(mn)}{\sqrt{n}} \,=\,  \delta\times\frac{\sqrt{d}\,\log(mn)}{\sqrt{n}} \,=\, \epsilon/2.
\end{equation}
Then, it follows from $\mathcal{E'''}$ that
\begin{equation}
\begin{split}
\|\hat{\nabla} F(p^*)\|&\leq \big\|\hat{\nabla} F(p^*)-\nabla F(p^*)\big\|\,+\,\|\nabla F(p^*)\|\\
&\leq\, \epsilon/4+\|\nabla F(p^*)\|\\
&=\,\epsilon/4+ \big\|\nabla F(p^*)-\nabla F(\theta^*)\big\|\\
&\leq\, \epsilon/4 +\big\|p^*-\theta^*\big\|\\
&\leq\, \epsilon/4+\epsilon/2\\
&=\, 3\epsilon/4,
\end{split}
\label{eq:11star}
\end{equation}
where the second inequality is due to $\mathcal{E'''}$, the third inequality follows from the Lipschitz continuity of $\nabla F$, and the last inequality is from \eqref{eq:p theta eps2}.
Therefore, 
\begin{align*}
\|\hat{\theta}-\theta^*\|\,&\leq\, \frac{1}{\lambda}\,\big\|\nabla F(\hat{\theta})-\nabla F(\theta^*)\big\|\\
&=\,\frac{1}{\lambda}\,\|\nabla F(\hat{\theta})\|\\
&\leq \,\frac{1}{\lambda}\,\|\hat{\nabla} F(\hat{\theta})\|+\frac{1}{\lambda}\,\big\|\hat{\nabla} F(\hat{\theta})-\nabla F(\hat{\theta})\big\|\\
&\leq^a\, \frac{1}{\lambda}\,\|\hat{\nabla} F(\hat{\theta})\|+\frac{\epsilon}{4\lambda}\\
& \leq^b\, \frac{1}{\lambda}\,\|\hat{\nabla} F(p^*)\|+\frac{\epsilon}{4\lambda}\\
&\leq^c\, \frac{3\epsilon}{4\lambda}+\frac{\epsilon}{4\lambda}\\
&=\,\frac{\epsilon}{\lambda},
\end{align*}
($a$) Due to event $\mathcal{E'''}$. 
\\ ($b$) Because the output of the server, $\hat{\theta}$, is a grid point $p$ in $\tilde{G}_{s^*}^t$  with  smallest $\|\hat{\nabla}F(p)\|$.
\\ ($c$) According to \eqref{eq:11star}.
\\ Finally, it follows from \eqref{eq:s* in D} and Lemma~\ref{lem:prob E3 upper} that the above inequality holds with probability $1-\exp\big(-\Omega (\log^2(mn))\big)$. 
Equivalently,
\begin{equation*}
\Pr\left(\|\hat{\theta}-\theta^*\| \ge \frac{\epsilon}{\lambda} \right)\,=\,\exp\Big(-\Omega\big(\log^2(mn)\big)\Big),
\end{equation*}
and Theorem~\ref{th:main upper c} follows.


\medskip
\section{Proof of Lemma~\ref{lem:pro E1 c}} \label{app:proof lem EM1}
Suppose that machine $i$ observes functions $f_1^i,\ldots,f_n^i$. 
Recall the definition of $\theta^i$ in \eqref{eq:def theta i half upper c}. 
The following proposition provides a bound on $\theta^i-\theta^*$, which improves upon the bound in Lemma~8 of \citep{zhang2013information}. 
\begin{claim}
	For any $i\le m$,
	\begin{equation*}
	\Pr\left(\|\theta^i-\theta^*\|\geq \frac{\alpha}{\sqrt{n}}\right)
	\,\leq\, d\exp\left( \frac{-\alpha^2 \lambda^2}{d} \right),
	\end{equation*}
	where $\lambda$ is the lower bound on the curvature of $F$ (cf. Assumption~\ref{ass:1}).
	\label{lemma:10}
\end{claim}
\begin{proof}[Proof of Claim]
Let $F^i(\theta)=\sum_{j=1}^{n/2} f_j^i(\theta)$, for all $\theta\in[-1,1]^d$.
From the lower bound $\lambda$ on the second derivative of $F$, we have  
$$\| \nabla F(\theta^i) - \nabla {F}^i(\theta^i)\| \,=\, \|\nabla F(\theta^i)\| \,=\, \|\nabla F(\theta^i) - \nabla F(\theta^*) \| \,\geq\, \lambda \|\theta^i-\theta^*\|, $$
where  the two equalities are because $\theta^i$ and $\theta^*$ are the the minimizers of ${F}^i $ and $F$, respectively.  
Then,
\begin{equation}
\begin{split}
\Pr\left(\|\theta^i-\theta^*\| \ge \frac{\alpha}{\sqrt{n}}\right) \,& \le \, 
\Pr\left(\| \nabla F(\theta^i) - \nabla {F}^i(\theta^i)\| \geq \frac{\lambda\alpha}{\sqrt{n}}\right)\\
& \le^{a}\, \sum_{j=1}^d \Pr\left(\Big|\frac{\partial {F}^i(\theta^i)}{\partial \theta_j} - \frac{\partial {F}(\theta^i) }{\partial \theta_j}\Big| > \, \frac{\alpha \lambda}{\sqrt{d}\sqrt{n}}\right)\\
& =\, d\,\Pr\left(\Big|\frac{2}{n}\sum_{l=1}^{n/2} \frac{\partial}{\partial \theta_j} f_l^i(\theta^i) \,-\,   \E_{f \sim P}\big[\frac{\partial}{\partial \theta_j} f(\theta^i) \big] \Big|\geq \frac{\alpha \lambda}{\sqrt{d}\sqrt{n}} \right)\\
& =^b \, d \exp\left(-\frac{\alpha^2 \lambda^2}{d}\right),
\end{split}
\end{equation}
\\($a$) Follows from the union bound and the fact that for any $d$-dimensional vector $v$, there exists an entry $v_i$ such that $\|v\|\le |v_i|/\sqrt{d}$.
\\($b$) Due to Hoeffding's inequality (cf. Lemma~\ref{lemma:CI}~(a)).\\
This completes the proof of Claim~\ref{lemma:10}.
\end{proof}

Based on Claim~\ref{lemma:10}, we can write
\begin{equation}\label{eq:union for theta i}
\begin{split}
\Pr\bigg(\|\theta^i-\theta^*\|&\leq \frac{\log(mn)}{2\sqrt{n}},\, \mbox{for } i=1,\ldots,m\bigg)\\ &\geq\,  1-m \Pr\left(\|\theta^1-\theta^*\|\geq \frac{\log(mn)}{2\sqrt{n}}\right)\\
&\geq 1-md \exp\left(\frac{-\lambda^2\log^2(mn)}{4d}\right) \\
&=\,1-\exp\Big(-\Omega\big(\log^2(mn)\big)\Big),
\end{split}
\end{equation}
where the first inequality is due to the union bound and the fact that the  distributions of $\theta^1,\ldots,\theta^m$ are identical, and the second inequality follows from Lemma~\ref{lemma:10}.
Thus, with probability at least $1-\exp\big(-\Omega(\log^2(mn))\big)$, every $\theta^i$ is  in the distance $\log(mn)/2\sqrt{n}$ from $\theta^*$. 
Recall that for each machine $i$, all sub-signals of machine $i$ have the same $s$-component. Hereafter, by the $s$-component of a machine we mean the $s$-component of the sub-signals generated at that machine. For each $i$, let $s^i$ be the $s$-component of machine $i$.
Therefore, with probability at least $1-\exp\big(-\Omega(\log^2(mn))\big)$, for any machine $i$,
\begin{align*}
\Pr\left(\|s^i-\theta^*\|_\infty > \frac{\log(mn)}{\sqrt{n}} \right)\,
&\le\, \Pr\left(\|s^i-\theta^i\|_\infty
+ \|\theta^i-\theta^*\|_\infty > \frac{\log(mn)}{\sqrt{n}} \right)\\
&\le\, \Pr\left(\|s^i-\theta^i\|_\infty  > \frac{\log(mn)}{2\sqrt{n}} \right) \\
&\qquad + \Pr\left( \|\theta^i-\theta^*\|_\infty > \frac{\log(mn)}{2\sqrt{n}} \right)\\
& = \, 0 +   \Pr\left( \|\theta^i-\theta^*\|_\infty > \frac{\log(mn)}{2\sqrt{n}} \right)\\
& = \,   \exp\Big(-\Omega\big(\log^2(mn)\big)\Big),
\end{align*}
where the first equality is due to the choice of  $s^i$ as the nearest grid point to $\theta^i$, and the last equality follows from \eqref{eq:union for theta i}.
Recall that $s^*$ is the grid point with the largest number of occurrences in the received signals.
Therefore, \eqref{eq:s* in D} is valid with probability at least $1-\exp\big(-\Omega(\log^2(mn))\big)$.
Equivalently, $\theta^*$ lies in the $\big(2\log(mn)/\sqrt{n}\big)$-cube $C_{s^*}$ centered at $s^*$. 

Since grid $G$ has block size $\log(mn)/\sqrt{n}$, there are at most $2^d$ points $s$ of the grid that satisfy
$\|s-\theta^*\|_\infty \leq \log(mn)/\sqrt{n}$.
It then follows from \eqref{eq:s* in D} that 
$
\Pr\big(\mathcal{E'}\big)=1-\exp\big(-\Omega(\log^2(mn))\big)
$. 
This completes the proof of Lemma~\ref{lem:pro E1 c}.



\medskip
\section{Proof of Lemma~\ref{lem:prob of E2 upper}} \label{app:proof lem E2 upper}
Suppose that the $s$-component of machine $i$ is $s = s^*$ and assume that $\mathcal{E'}$ is valid. We begin with a simple inequality:
for any $x\in[0,1]$ and any $k>0$,
\begin{equation}\label{eq:power exp min}
1-(1-x)^k \, \ge\, 1- e^{kx} \, \ge\, \frac12 \min\big(kx,1\big).
\end{equation}

Let $Q_{p}$ be the  probability that $p$ appears in the $p$-component of at least one of the sub-signals of machine $i$.
Then, for $p \in \tilde{G}_{s^*}^l$,
\begin{equation*} \label{eq:Qp and qp for p}
\begin{split}
Q_{p} \,&=\, 1-\left(1-2^{-dl}\times\frac{2^{(d-2)l}}{\sum_{j=0}^t 2^{(d-2)j}}\right)^{\lfloor B/(d\log mn)  \rfloor} \\
\, &\ge\, \frac12 \min\left(\frac{2^{-2l}\,\big\lfloor B/(d\log mn)  \big\rfloor}{\sum_{j=0}^t 2^{(d-2)j}},\,1\right)\\
\, &\ge\, \frac12 \min\left(\frac{2^{-2l} B}{d\log (mn) \,\sum_{j=0}^t 2^{(d-2)j}},\,1\right)
\end{split}
\end{equation*}
where the equality is due to the probability of a point $p$ in   $\tilde{G}_{s^*}^l$ (see \eqref{eq:prob choose p c})  and the number  $\lfloor B/(d\log mn)  \rfloor$  of sub-signals per machine, and the first inequality is due to \eqref{eq:power exp min}. 
Assuming $\mathcal{E'}$, we then have 
\begin{equation}
\begin{split}
\mathbb{E}\big[N_{p}\big]\,&=\,  Q_{p} m^*
\,\geq\, \min\left(\frac{2^{-2l} m B}{2^{d+1}d\log (mn) \, \sum_{j=0}^t 2^{(d-2)j}},\, \frac{m}{2^{d+1}}\right).
\end{split}
\label{eq:8star 1}
\end{equation}
We now  bound the two terms on the right hand side of \eqref{eq:8star 1}.
For the second term on the right hand side of \eqref{eq:8star 1}, we have
\begin{equation}
\begin{split}
\frac{m}{2^{d+1}} \,& =\, \frac{m\epsilon^2}{2^{d+1}\epsilon^2} \\
\,&\ge\,\frac{16 \, m d^3 \log^{8}(mn) 2^d}{2^{d+1} mn \epsilon^2} \\
\,&>\, \frac{8 d^2\log^6(mn)}{n\epsilon^2},
\end{split}
\label{eq:8star 3}	
\end{equation}
where the first inequality is from the definition of $\epsilon$ in \eqref{eq:def eps upper}.
For the first term at the right hand side of \eqref{eq:8star 1},
if $d\le2$, then
\begin{equation} \label{eq:frac 1 delta to the max d 2 1 c}
\left(\frac{1}\delta\right)^{\max(d,2)} \,=\, \left(\frac1{\delta}\right)^2 \,\le\, \frac{m}{4 d^2 \log^6(mn)\, 2^d},\,<\, \frac{mB}{4 d^2 \log^6(mn)\, 2^d}.
\end{equation}
On the other hand, if $d\ge 3$, then
\begin{equation}\label{eq:frac 1 delta to the max d 2 2 c}
\left(\frac{1}\delta\right)^{\max(d,2)} \,=\, \left(\frac1{\delta}\right)^d \,\le\, \frac{mB}{2^d d^{d} \log^{3d}(mn)}\,<\, \frac{mB}{4 d^2 \log^6(mn)\, 2^d}.
\end{equation}
Moreover, 
\begin{equation}\label{eq:1delta vs m}
1/\delta \,\le\, \frac1{2d\log^3 mn} \times\frac{\sqrt{m}}{2^{d/2}}\,<\, m.
\end{equation}
It follows that for any $d\ge 1$,
\begin{equation*}\label{eq:t 2tmax bound}
\begin{split}
\sum_{j=0}^t 2^{(d-2)j} \,&\le\, t  2^{t\max(d-2,0)}\\ \,&\le\, \log(mn)\, 2^{t\max(d-2,0)}\\
&=\, \log(mn)\, \left(\frac{1}\delta\right)^{\max(d-2,0)}\\
&=\, \log(mn)\,\delta^2\, \left(\frac{1}\delta\right)^{\max(d,2)}\\
&\le\, \log(mn)\,\delta^2\, \frac{mB}{4 d^2 \log^6(mn)\, 2^d}\\
&=\, \log(mn)\times \frac{n\epsilon^2}{4d\log^2{(mn)}} \times \frac{mB}{4 d^2 \log^6(mn)\, 2^d}\\
&=\, \frac{nmB\epsilon^2}{16 d^3 \log^7(mn)\, 2^d },
\end{split}
\end{equation*}
where the second inequality is due to \eqref{eq:1delta vs m} and $t=\log (1/\delta )$. third inequality follows from \eqref{eq:frac 1 delta to the max d 2 1 c} and \eqref{eq:frac 1 delta to the max d 2 2 c}, the third equality is from the definition of $\epsilon$ in \eqref{eq:def eps upper}.
Then,
\begin{equation} \label{eq:mb over n sigma bound c}
\begin{split}
\frac{2^{-2l} m B}{2^{d+1}d\log (mn) \, \sum_{j=0}^t 2^{(d-2)j}}\,&\ge\, 
\frac{2^{-2l} m B}{2^{d+1}d\log (mn) } \times \frac{16 d^3 \log^7(mn)\, 2^d}{nmB\epsilon^2}\\
&=\,\frac{8 d^2\log^6(mn)\, 2^{-2l}}{n\epsilon^2}.
\end{split}
\end{equation}
Plugging \eqref{eq:8star 3} and \eqref{eq:mb over n sigma bound c} into \eqref{eq:8star 1}, it follows that for  $l=0,\ldots,t$ and for any $p\in \tilde{G}_{s^*}^l$,
\begin{equation}\label{eq:8stat s}
\mathbb{E}\big[N_{p}\big] \,\ge\, \frac{8 d^2\log^6(mn)\, 2^{-2l}}{n\epsilon^2}.
\end{equation}	
Given the bound in \eqref{eq:8stat s}, for $l=0,\ldots, t$, we have
\begin{equation}\label{eq:frac to log 2}
\begin{split}
\frac18 \mathbb{E}\big[N_{p}\big] \,&\ge\,\frac{d^2\log^6(mn)\, 2^{-2l}}{n\epsilon^2} \\
&\, \ge\, \frac{d^2 \log^6(mn)\,2^{-2t}}{n\epsilon^2} \\
&\, =\, \frac{d^2 \log^6(mn)\,\delta^2}{n^2\epsilon^2} \\
&\, =\, \frac{d^2 \log^6(mn)\,\delta^2}{4 d \log^2 (mn)\, \delta^2} \\
&\,=\, \frac{d\log^4 (mn)}4, 
\end{split}
\end{equation}
where the first equality is from definition of $t$ and the second equality is by the definition of $\epsilon$ in \eqref{eq:def eps upper}.
Then, for any $l\in0,\ldots,t$ and any $p\in \tilde{G}_{s^*}^l$,
\begin{equation}\label{eq:pr NP log4}
\begin{split}
\Pr\left(N_p\leq \frac{4 d^2\log^6(mn)\,2^{-2l}}{n\epsilon^2}\right) \,&\leq\, \Pr\left(N_p\leq \frac{\mathbb{E}[N_p]}{2}\right)\\
\,& \leq \,2^{-(1/2)^2\mathbb{E}[N_p]/2}\\
\,&\le\,2^{-d\log^4(mn)/4}\\
\,&<\,\exp\big({-d\log^4(mn)/6}\big),
\end{split}
\end{equation}
where the first inequalities are due to \eqref{eq:8stat s}, Lemma~\ref{lemma:CI}~(b),  \eqref{eq:frac to log 2}, and the fact that $\ln(2)/4> 1/6$ respectively.
Then,
\begin{align*}
\Pr\big(\mathcal{E''} \mid \mathcal{E'}\big)\, 
&=\, \Pr\left(N_p\geq \frac{4 d^2\log^6(mn)\,2^{-2l}}{n\epsilon^2}, \quad  \forall p\in \tilde{G}_{s^*}^l\mbox{ and for }l=0,\ldots,t\right)\\
&\ge\,
1- \sum_{l=0}^{t} \sum_{p\in \tilde{G}_{s^*}^l} \Pr\left(N_p\leq \frac{4 d^2\log^6(mn)\,2^{-2l}}{n\epsilon^2}\right)\\
&\geq\, 1- t2^{dt} \exp\big(-d\log^4(mn)/6\big)\\
&=\, 1- \log(1/\delta)\left(\frac1\delta\right)^d \exp\big(-d\log^4(mn)/6\big)\\
&\ge\, 1- \log(m)\,m^d \exp\big(-d\log^4(mn)/6\big)\\
&=\,1-\exp\Big(-\Omega\big(\log^4(mn)\big)\Big),
\end{align*}
where the first equality is by the definition of $\mathcal{E''}$, the first inequality is from union bound, the second inequality due to \eqref{eq:pr NP log4}, and the third equality follows from \eqref{eq:1delta vs m}.
It then follows from Lemma~\ref{lem:pro E1 c} that $\Pr\big(\mathcal{E''} \big)= 1-\exp\big(-\Omega(\log^2(mn))\big)$ and Lemma~\ref{lem:prob of E2 upper} follows.

\medskip
\section{Proof of Lemma~\ref{lem:prob E3 upper}} \label{app:proof E3 upper}
For any $l\le t$ and any $p\in \tilde{G}_{s^*}^0$, let
$$\hat{\Delta}(p)= \frac{1}{N_p}\,\, \sum_{\substack{\mbox{\scriptsize{Subsignals of the form }} \\ (s^*,p,\Delta)\\ \mbox{\scriptsize{after redundancy elimination}}}} \Delta,$$
and let $\Delta^*(p)= \mathbb{E}[\hat{\Delta}(p)]$.

We first consider the case of $l=0$. Note that $\tilde{G}_{s^*}^0$ consists of a single point $p=s^*$.
Moreover, the component $\Delta$ in each signal is the average over the gradient of $n/2$ independent functions. 
Then, $\hat{\Delta}(p)$ is the average over the gradient of $N_p \times n/2$ independent functions.
Given event $\mathcal{E''}$, for any entry $j$ of the gradient, it follows from Hoeffding's inequality (Lemma~\ref{lemma:CI}~(a)) that
\begin{equation}
\begin{split}
\Pr&\left(\big|\hat{\Delta}_j(s^*)-\Delta^*_j(s^*)\big|\geq \frac{\epsilon}{4\sqrt{d}\log(mn)}\right)\\
&\qquad\leq\, \exp\left(-N_{s^*}n \times \left(\frac{\epsilon}{4\sqrt{d}\log(mn)}\right)^2 \,/\,2^2 \right)\\
&\qquad\leq \exp\left(-n {\frac{4d^2\log^6(mn)}{n\epsilon^2}} \times\frac{\epsilon^2}{16\, d\log^2(mn)}\right)\\
&\qquad=\exp\left(\frac{-d\log^4(mn)}{4}\right)\\
&\qquad=\, \exp\Big(-\Omega\big(\log^4(mn)\big)\Big).
\end{split}
\label{eq:10star}
\end{equation}

For $l\geq 1$, consider a grid point $p\in \tilde{G}_{s^*}^l$ and let $p'$ be the parent of $p$.
Then, $\|p-p'\|={\sqrt{d}\,2^{-l}\log(mn)}/\sqrt{n}$.
Furthermore, for any function $f\in \F$, we have $\|\nabla f(p)-\nabla f(p')\|\leq \|p-p'\|$. Hence, for any $j\le n$,
\begin{equation*}
\Big|\frac{\partial f(p)}{\partial x_j}-\frac{\partial f(p')}{\partial x_j}\Big|
\,\leq\, \|p-p'\|
\,=\,\frac{\sqrt{d}\log(mn)2^{-l}}{\sqrt{n}}.
\end{equation*}
Therefore, $\hat{\Delta}_j(p)$ is the average of $N_p \times n/2$ independent variables with absolute values no larger than $\gamma\triangleq {\sqrt{d}\log(mn)2^{-l}}/{\sqrt{n}}$. 
Given event $\mathcal{E''}$, it then follows from the Hoeffding's inequality that
\begin{align*}
\Pr&\left(\big|\hat{\Delta}_j(p)-\Delta_j^*(p)\big|
\geq\, \frac{\epsilon}{4\sqrt{d}\log(mn)}\right)\\
&\leq\, \exp\left(-{nN_p}\times\frac1{(2\gamma)^2}\times \left(\frac{\epsilon}{4\sqrt{d}\log(mn)}\right)^2\right)\\
&\leq\, \exp\left(-n \frac{4d^2 2^{-2l}\log^6(mn)}{n\epsilon^2}\times \frac{n}{4d2^{-2l}\log^2(mn)}\times \frac{\epsilon^2}{16 d \log^2(mn)}\right)\\
&=\,\exp\big(-n\log^2(mn)/16\big)\\
&=\,\exp\Big(-\Omega\big(\log^2(mn)\big)\Big),
\end{align*}
where the second inequality is by substituting $N_p$ from \eqref{eq:bound Np}. 
Employing union bound, we obtain
\begin{align*}
\Pr&\left(\big\|\hat{\Delta}(p)-\Delta^*(p)\big\| \ge \frac{\epsilon}{4\log(mn)}\right)\\
&\qquad\leq\, \sum_{j=1}^d\Pr\left(\big|\hat{\Delta}_j(p)-\Delta_j^*(p)\big|\geq \frac{\epsilon}{4\sqrt{d}\log(mn)}\right)\\
&\qquad =\, d \, \exp\Big(-\Omega\big(\log^2(mn)\big)\Big)\\
&\qquad=\exp\Big(-\Omega\big(\log^2(mn)\big)\Big).
\end{align*}
Recall from \eqref{eq:page17 c} that for any non-zero $l\leq t$ and any $p \in \tilde{G}_{s^*}^l$ with parent $p'$,
$$\hat{\nabla} F(p)- \nabla F(p)\,=\,\hat{\nabla} F(p')-\nabla F(p')+\hat{\Delta}(p)-\Delta^*(p).$$
Then,
\begin{align*}
\Pr&\left(\big\|\hat{\nabla} F(p)-\nabla F(p)\big\|>\frac{(l+1)\epsilon}{4\log(mn)}\right)\\
&\leq\,\Pr\left(\big\|\hat{\nabla} F(p')-\nabla F(p')\big\|>\frac{l\epsilon}{4\log(mn)}\right)\\
&\qquad +\Pr\left(\big\|\hat{\Delta}(p)-\Delta^*(p)\big\|>\frac{\epsilon}{4\log(mn)}\right)\\
&\leq\,\Pr\left(\big\|\hat{\nabla} F(p')-\nabla F(p')\big\|>\frac{l\epsilon}{4\log(mn)}\right) +\exp\Big(-\Omega\big(\log^2(mn)\big)\Big).
\end{align*}
Employing an induction on $l$, we obtain for any $l\le t$,
$$\Pr\left(\big\|\hat{\nabla} F(p)-\nabla F(p)\big\|>\frac{(l+1)\epsilon}{4\log(mn)}\right)
\,\le\,\exp\Big(-\Omega\big(\log^2(mn)\big)\Big).$$
Therefore, for any grid point $p$, 
\begin{align*}
\Pr\left(\big\|\hat{\nabla} F(p)-\nabla F(p)\big\|>\frac{\epsilon}{4}\right)\,
&\le\,\Pr\left(\big\|\hat{\nabla} F(p)-\nabla F(p)\big\|>\frac{(t+1)\epsilon}{4\log(mn)}\right)\\
&=\exp\Big(-\Omega\big(\log^2(mn)\big)\Big),
\end{align*} 
where the inequality is because $t+1 = \log(1/\delta)+1\le \log(mn)$.
It then follows from the union bound that
\begin{equation} \label{eq:E''' given E''}
\begin{split}
\Pr\big(\mathcal{E'''} \mid \mathcal{E''}\big)\, &\geq\,  1- \sum_{l=0}^{t} \sum_{p\in \tilde{G}_{s^*}^l}\Pr\left(\big\|\hat{\nabla} F(p)-\nabla F(p)\big\|>\frac{\epsilon}{4}\right)\\
&\geq\, 1- t2^{dt} \exp\Big(-\Omega\big(\log^2(mn)\big)\Big)\\
&=\, 1- \log(1/\delta)\left(\frac1\delta\right)^d \exp\Big(-\Omega\big(\log^2(mn)\big)\Big)\\
&\ge \,1-m\log(m)\exp\Big(-\Omega\big(\log^2(mn)\big)\Big)\\
&=\,1- \exp\Big(-\Omega\big(\log^2(mn)\big)\Big).
\end{split}
\end{equation} 
On the other hand, we have from Lemma~\ref{lem:prob of E2 upper} that $\Pr\big(\mathcal{E''} \big)= 1-\exp\big(-\Omega(\log^2(mn))\big)$.
Then, $\Pr\big(\mathcal{E'''} \big)= 1-\exp\big(-\Omega(\log^2(mn))\big)$ and Lemma~\ref{lem:prob E3 upper} follows.


\medskip

\section{Proof of Theorem~\ref{th:Hinf}} \label{sec:proof constant bit lower bound} 
Let $\Ft$ be a sub-collection of functions in $\F$ that are $\lambda$-strongly convex.
Consider  $2^B+2$ convex functions in $\Ft$: 
\begin{equation*}
f(\theta,i) \,\triangleq\,\theta^2+\frac{\theta^i}{i!},\qquad \textrm{for }\quad  \theta\in [-1,1] \quad\textrm{ and }\quad i=1,\ldots,2^B+2.
\end{equation*}
Consider a probability distribution $P$ over these functions that, for each $i$, associates probability $p_i$ to function $f(\cdot,i)$.
With an abuse of the notation, we use  $P$ also for a vector with entries $p_i$.
Since $n=1$, each machine observes only  one of $f_i$'s and it can send a $B$-bit length signal out of $2^B$ possible messages of length $B$ bits. 
As a general randomized strategy, suppose that each machine sends $j$-th message with probability $a_{ij}$ when it observes function $f(\cdot,i)$. 
Let $A$ be a ${(2^B+2)\times 2^B}$ matrix with entries $a_{ij}$. 
Then, each machine sends $j$-th message with probability $\sum_i p_i a_{ij}$. 

At the server side, we only observe the number (or frequency) of occurrences of each message. 
In view of the law of large number, as $m$ goes to infinity, the frequency of $j$-th message  tends to $\sum_i p_ia_{ij}$, for all $j\leq 2^B$.
Thus, in the case of infinite number of machines, the entire information of all transmitted signals is captured in the vector $A^TP$. 

Let $\hat{G}$ denote the estimator located in the server, that takes  the vector  $A^TP$  and outputs an estimate  $\hat{\theta}=\hat{G}(A^TP)$  of the minimizer of $F(\theta)=\mathbb{E}_{x\sim P}\big[f(\theta,x)\big]$.
We also let  $\theta^*=G(P)$ denote the optimal solution (i.e., the minimizer of $F$).
In the following, we will show that the expected error $\mathbb{E}\big[|\hat{\theta}-\theta^*|\big]=\mathbb{E}\big[|\hat{G}(A^TP)-G(P)|\big]$ is lower bounded by a universal constant, for all matrices $A$ and all estimators $\hat{G}$.

We say that vector $P$ is central if 
\begin{equation}
\sum_{i=1}^{2^B+1} p_i=1, \qquad\textrm{and}\qquad p_i\geq \frac{1}{2^B+2}, \quad \textrm{for} \quad i=1,\cdots,2^B+1.
\label{eq:central}
\end{equation}
Let $\mathcal{P}_c$ be the collection of central vectors $P$. We define two constants 
\begin{align*}
\theta_1\,&\triangleq\, \inf_{P\in \mathcal{P}_c} \argmin{\theta} \sum_{i=1}^{2^B+2} p_i f(\theta,i),\\
\theta_2\,&\triangleq\, \sup_{P\in \mathcal{P}_c} \argmin{\theta} \sum_{i=1}^{2^B+2} p_i f(\theta,i).
\end{align*}
For any central $P$, the minimizer of $\mathbb{E}_{x\sim P}[f(\theta,x)]$ lies in the interval $[\theta_1,\theta_2]$. 
Furthermore, since  functions $f(\cdot,1)$ and $f(\cdot,2)$ have different minimizers, we have $\theta_1 \ne \theta_2$. 
Let
\begin{equation} \label{eq:def eps const}
\epsilon\triangleq \inf_{\substack{v\in \mathbb{R}^{2^b+2}\\\|v\|=1}} \, \,\sup_{\theta\in [\theta_1,\theta_2]} \,\Bigg|\sum_{i=1}^{2^B+2}v_i \, f'(\theta,i)\Bigg|,
\end{equation}
where $f'(\theta,i)= d/d\theta\, f(\theta,i)$.
We now show that $\epsilon>0$. 
In order to draw a contradiction, suppose that $\epsilon =0$. 
In this case, there exists nonzero vector $v$ such that the polynomial $\sum_{i=1}^{2^B+2} v_i f'(\theta,i)$ is equal to zero for all $\theta\in [\theta_1,\theta_2]$.
On the other hand, it follows from the definition of $f(\cdot,i)$ that for any nonzero vector $v$, $\sum_{i=1}^{2^B+2} v_i f'(\theta,i)$ is a nonzero polynomial of degree no larger than $2^B+1$. 
As a result, the fundamental theorem of algebra \citep{krantz2012handbook}  implies that this polynomial has at most $2^B+1$ roots and it cannot be zero over the entire interval $[\theta_1,\theta_2]$. 
This contradict with the earlier statement that the polynomial of interest equals zero throughout the interval $\theta\in [\theta_1,\theta_2]$.
Therefore,  $\epsilon>0$. 

Let $v$ be a vector of length $2^B+2$ such that $A^Tv=0$, $\|v\|=1$, and $\sum_i v_i=0$. 
Note that such $v$ exists and lies in the null-space of matrix $[A|\mathbf{1}]^T$, where $\mathbf{1}$ is the vector of all ones. 
Let $\theta'$ be the solution of the following optimization problem 
$$\theta'\,=\,\argmax{\theta\in [\theta_1,\theta_2]} \,\,\Bigg|\sum_{i=1}^{2^B+2}v_i f'(\theta,i)\Bigg|,$$
and assume that $P$ is a central vector such that $G(P)=\theta'$. Then, it follows from \eqref{eq:def eps const} that 
\begin{equation}
\Bigg|\sum_{i=1}^{2^B+2}v_i f'(\theta',i)\Bigg|\geq \epsilon.
\label{eq:eps}
\end{equation}

Let $Q=P+2^{-(B+2)}v$. Then, from the conditions in \eqref{eq:central} and $\|v\|=1$, we can conclude that $Q$ is a probability vector. 
Furthermore, based on the definition of $v$, 
\begin{equation}
A^TQ=A^TP+A^Tv=A^TP.
\label{eq:A}
\end{equation}
It then follows from \eqref{eq:eps} that
\begin{equation}
\Bigg|\frac{d}{d\theta} \mathbb{E}_{i\sim Q} [f(\theta,i)]\Big|_{\theta=\theta'}\Bigg|
\,=\,\Bigg|\sum_{i=1}^{2^B+2} \left(p_i+\frac{v_i}{2^{B+2}}\right)f'(\theta',i)\Bigg|
\,=\,\frac{1}{2^{B+2}}\Bigg|\sum_{i=1}^{2^B+2} v_i f'(\theta',i)\Bigg|
\,\geq\, \frac{\epsilon}{2^{B+2}},
\label{eq:eps2}
\end{equation} 
where the last equality is due to the fact that $\theta'$ minimizes $\mathbb{E}_{i\sim P} [f(\theta,i)]$. 

Let $\theta''=G(Q)$ be the  minimizer of $\mathbb{E}_{i\sim Q} [f(\theta,i)]$. 
Then, 
\begin{equation}\label{eq:d/dt E=0}
\frac{d}{dt}\, \mathbb{E}_{i\sim Q} [f(\theta,i)]\big|_{\theta=\theta''}=0.
\end{equation}
Furthermore, for any $ i \leq 2^B+2$ and any $\theta \in [-1,1]$, its easy to see that $|f''(\theta,i)|\leq 4$.
Consequently,   $\big|d^2/d\theta^2 \,\mathbb{E}_{i\sim Q} [f(\theta,i)]\big|\leq 4$, for all $\theta \in [-1,1]$. 
It follows that 
\begin{align*}
\nonumber|G(Q)-G(P)|\,&=\,|\theta''-\theta'| \\
&\geq\,\frac14\, \Big|\frac{d}{d\theta} \mathbb{E}_{i\sim Q} [f(\theta,i)]\big|_{\theta=\theta''}- \frac{d}{d\theta} \mathbb{E}_{i\sim Q} [f(\theta,i)]\big|_{\theta=\theta'}\Big|\\
&=\,\frac14\, \Big|\frac{d}{d\theta} \mathbb{E}_{i\sim Q} [f(\theta,i)]\big|_{\theta=\theta'}\Big|\\
&\geq\, \frac{\epsilon}{2^{B+4}},
\label{eq:PQ}
\end{align*}
where the last two relations are due to \eqref{eq:d/dt E=0} and \eqref{eq:eps2}, respectively.
Then,
\begin{align*}
\big|\hat{G}(A^TP)-G(P)\big|+\big|\hat{G}(A^TQ)-G(Q)\big|\,
&\geq\, \big|G(Q)-G(P)+\hat{G}(A^TP)-\hat{G}(A^TQ)\big|\\
&=\, \big|G(Q)-G(P)\big|\\
&\geq\, \frac{\epsilon}{2^{B+4}},
\end{align*}
where the equality follows from  \eqref{eq:A}. Therefore, 
the estimation error exceeds $\epsilon/2^{B+5}$ for at least one of the probability vectors $P$ or $Q$. 
This completes the proof of Theorem~\ref{th:Hinf}.


\medskip
\section{Proof of Theorem~\ref{th:H2constb}} \label{app:constant bit upper bound}
For simplicity, in this proof we will be working with the $[0,1]^d$ cube as the domain.
Consider the following randomized algorithm: 
\begin{itemize}
	\item Suppose that each machine $i$ observes $n$ function $f_1^i,\cdots,f_n^i$ and finds the minimizer  of $\sum_{j=1}^n f_j^i(\theta)$, which we denote by $\theta^i$. 
	Machine $i$ then lets its signal $Y^i$ be a randomized binary string of length $d$ of the following form: for $j=1,\ldots,d$,
	\begin{equation*}
	Y_j^i=\begin{cases}
	0, \qquad\mbox{with probability    } \theta^i_j,\\
	1, \qquad \mbox{with probability    } 1-\theta^i_j, 
	\end{cases}
	\end{equation*}
	where $Y_j^i$ is the $j$-th bit of $Y^i$, and $\theta^i_j$ is the $j$-th entry of $\theta^i$.
	\item The server receives signals from all machines and outputs $\hat{\theta}=1/m\sum_{i=1}^m Y^i$.
\end{itemize}
For the above algorithm, we have for $j=1,\ldots,d$,
\begin{equation} \label{eq:var of theta hat}
\operatorname{var}\big(\hat{\theta}_j\big)\,=\,\operatorname{var}\Bigg(\frac{1}{m} \sum_{i=1}^m Y^i_j\Bigg)
\,=\,\frac{1}{m}\operatorname{var}\big(Y^1_j\big)
\,=\,O\left(\frac{1}{m}\right),
\end{equation}
where the last equality is because $Y^1_j$ is a binary random variable. 
Then, 
\begin{align*}
\mathbb{E}\Big[\|\hat{\theta}-\theta^*\|^2\Big]
\,&=\, \sum_{j=1}^{d}\,  \mathbb{E}\Big[\big(\hat{\theta}_j-\theta^*_j\big)^2\Big]\\
\,&=\,\sum_{j=1}^{d}  \mathbb{E}\Big[\big(\hat{\theta}_j-\mathbb{E}[\hat{\theta}_j]\,+\,
 \mathbb{E}[\hat{\theta}_j]-\theta^*_j\big)^2\Big]\\
&=\, \sum_{j=1}^{d} \mathbb{E}\Big[\big(\hat{\theta}_j-\mathbb{E}[\hat{\theta}_j]\big)^2\Big]
\,+\, \sum_{j=1}^{d} \mathbb{E}\Big[\big(\mathbb{E}[\hat{\theta}_j]-\theta^*_j\big)^2\Big]\\
&=\,\sum_{j=1}^{d} \operatorname{var}\big(\hat{\theta}_j\big)
\,+\,\sum_{j=1}^{d}  \Big(\mathbb{E}\big[\hat{\theta}_j-\theta^*_j\big]\Big)^2\\
&=\, O\left(\frac{d}{m}\right)+O\left(\frac{d}{n}\right),
\end{align*}
where the last equality is due to \eqref{eq:var of theta hat} and Claim~\ref{lemma:10}. 
This completes the proof of Theorem~\ref{th:H2constb}.

\end{document}